%% file: main.tex
\documentclass[letterpaper]{article} 
\usepackage{aaai2026}  
\usepackage{times}  
\usepackage{helvet}  
\usepackage{courier}  
\usepackage[hyphens]{url}  
\usepackage{graphicx} 
\usepackage{natbib}  
\usepackage{caption} 
\usepackage{algorithm}
\usepackage{algorithmic}

\usepackage{newfloat}
\usepackage{listings}
\DeclareCaptionStyle{ruled}{labelfont=normalfont,labelsep=colon,strut=off} 
\floatstyle{ruled}
\newfloat{listing}{tb}{lst}{}
\floatname{listing}{Listing}
\title{Indirect Query Bayesian Optimization with Integrated Feedback}
\author{
    Mengyan Zhang\textsuperscript{\rm 1},
    Shahine Bouabid\textsuperscript{\rm 2},
    Cheng Soon Ong\textsuperscript{\rm 3},
    Seth Flaxman\textsuperscript{\rm 1},
    Dino Sejdinovic\textsuperscript{\rm 4}
}
\affiliations{
    \textsuperscript{\rm 1}University of Oxford,
    \textsuperscript{\rm 2}Massachusetts Institute of Technology,\\
    \textsuperscript{\rm 3}CSIRO and Australian National University,
    \textsuperscript{\rm 4}University of Adelaide

}

\usepackage{algorithm}
\usepackage{algorithmic}

\usepackage{amsthm}
\usepackage{amssymb}
\usepackage{amsmath}
\theoremstyle{plain}
\newtheorem{theorem}{Theorem}[section]

\newtheorem{lemma}[theorem]{Lemma}

\theoremstyle{definition}
\newtheorem{definition}[theorem]{Definition}
\newtheorem{assumption}[theorem]{Assumption}
\theoremstyle{remark}
\newtheorem{remark}[theorem]{Remark}

\usepackage{bm}

\input{lettres_maths}

\usepackage[textsize=tiny]{todonotes}

\usepackage{newfloat}
\usepackage{listings}
\DeclareCaptionStyle{ruled}{labelfont=normalfont,labelsep=colon,strut=off} 
\floatstyle{ruled}
\newfloat{listing}{tb}{lst}{}
\floatname{listing}{Listing}
\usepackage{bibentry}

\begin{document}

\maketitle


\begin{abstract}
We develop the framework of \textit{Indirect Query Bayesian Optimization (IQBO)}, a new class of Bayesian optimization problems where the integrated feedback is given via a conditional expectation of the unknown function $f$ to be optimized. The underlying conditional distribution can be unknown and learned from data. The goal is to find the global optimum of $f$ by adaptively querying and observing in the space transformed by the conditional distribution. This is motivated by real-world applications where one cannot access direct feedback due to privacy, hardware or computational constraints. We propose the Conditional Max-Value Entropy Search (CMES) acquisition function to address this novel setting, and propose a hierarchical search algorithm with multi-resolution feedback to improve computational efficiency. We show regret bounds for our proposed methods and demonstrate the effectiveness of our approaches on simulated optimization tasks.
\end{abstract}


\section{Introduction}

Bayesian optimization (BO) is a sequential decision-making method for global optimization of black-box functions \citep{garnett_bayesopIQBOok_2023}.
BO has been widely studied in machine learning and applied in real-world applications such as optimising hyperparameters \citep{wu2019hyperparameter} and experimental design in biology \citep{pmlr-v162-stanton22a}.
We formalize a new class of BO problems called \textit{Indirect Query Bayesian Optimization (IQBO)}, where instead of querying directly from the input space $x \in \mathcal{X}$, one has to select from a transformed space $a \in \mathcal{A}$ linked to the input space via a potentially unknown conditional distribution $p(x|a)$. Feedback consists of the noisy conditional expectation of the target function, $\mathbb{E}_{X}[f(X)|A=a] + \epsilon$ with $\epsilon \sim \mathcal{N}(0, \sigma^2)$. The goal of IQBO is to optimize the target function $f$ under such integrated feedback using a given budget of queries. 

The task of IQBO is motivated by real-world applications where only access to a transformed space and corresponding feedback are possible, while one is still interested in optimizing the target function in the original input space. For example, in multi-armed bandits, we may not be able to directly query an arm, but only the average reward given a context with a particular policy; in treatment experimental design, the actual treatment applied to the patient can differ from what the doctor recommends \cite{pmlr-v70-hartford17a}; in epidemiology, when finding the highest Malaria incidence rates, the available observations are only aggregated incidence counts over a larger geographical region \citep{law_variational_2018}.
We further illustrate our motivations in Section
\ref{sec: problem setup and cmp}.

The IQBO problem is challenging, since we can only indirectly query the objective function. Moreover, the feedback is integrated via potentially unknown conditional expectations.
Our approach models the target function $f$ as a Gaussian process. We can then leverage the framework of Conditional Mean Processes (CMP) \citep{chau_deconditional_2021}, developed for the problem of statistical downscaling, to update our beliefs on $f$ using conditional mean observations, giving the corresponding posterior Gaussian process.
However, the key challenge of IQBO remains: \textit{how to use this posterior to design the appropriate query strategy given that the observation and recommendation spaces are inherently mismatched?}  
The classical methods, such as Upper Confidence Bounds (UCB) or Thompson Sampling (TS), cannot be directly applied due to the mismatch, and as a result we require a bespoke acquisition function.

We develop a one-step look-ahead information gain approach, proposing the Conditional Max-value Entropy Search (CMES) policy (Section \ref{sec: CMES}), 
where the agent selects the data point with the largest mutual information between the optimum value of the target function and the corresponding integrated feedback. The CMES policy allows us to query in a transformed space but optimize the target function $f$ directly. 
The IQBO problem contains the multi-resolution conditional expectations as a special case (Section \ref{sec: motivation Multi-fidelity}), where higher resolution queries carry a higher cost.
We propose a tree search algorithm  (Section \ref{sec:hierarchical}) with the multi-resolution feedback, where we weight the CMES policy by the inverse cost and adaptively partition the space via a tree structure.

The final challenge is that none of the theoretical results including regret bounds, apply to the IQBO setting, and it is not clear under which assumptions on the feedback mechanism we can guarantee performance comparable to that of the direct feedback. 
We develop several theoretical results on IQBO in Section \ref{sec: theory}, including the instant regret upper bound for CMES policy in Theorem \ref{theo: Instant Regret Bound CMES} and for budgeted multi-resolution setting in Theorem \ref{theo: instant regret bound gpucb} under certain assumptions of conditional distribution. 
Simulated experiments in Section \ref{sec: Experiments} demonstrate that our proposed policy and algorithms outperform baselines.





Our \textbf{contributions} are as follows:
1) We introduce a novel problem of Indirect Query Bayesian Optimization with integrated feedback, highlighting its motivations and potential applications in various domains.
2) We develop a new acquisition function called Conditional Max-Value Entropy Search (CMES) specifically designed to address the new setting, ensuring efficient exploration of the search space.
3) We present a hierarchical search algorithm based on CMES to enhance computational efficiency and address multi-resolution settings. 
4) We establish regret upper bounds of the proposed CMES policy and budgeted multi-resolution setting, providing theoretical guarantees for its performance.
5) We demonstrate the empirical effectiveness of our proposed algorithm on synthetic optimization functions, comparing it to state-of-the-art baselines to showcase its superior performance.

\section{Problem Setup and Motivations}
\label{sec: problem setup and cmp}

\begin{figure}[t!]
    \centering
    \includegraphics[width=0.9\linewidth]{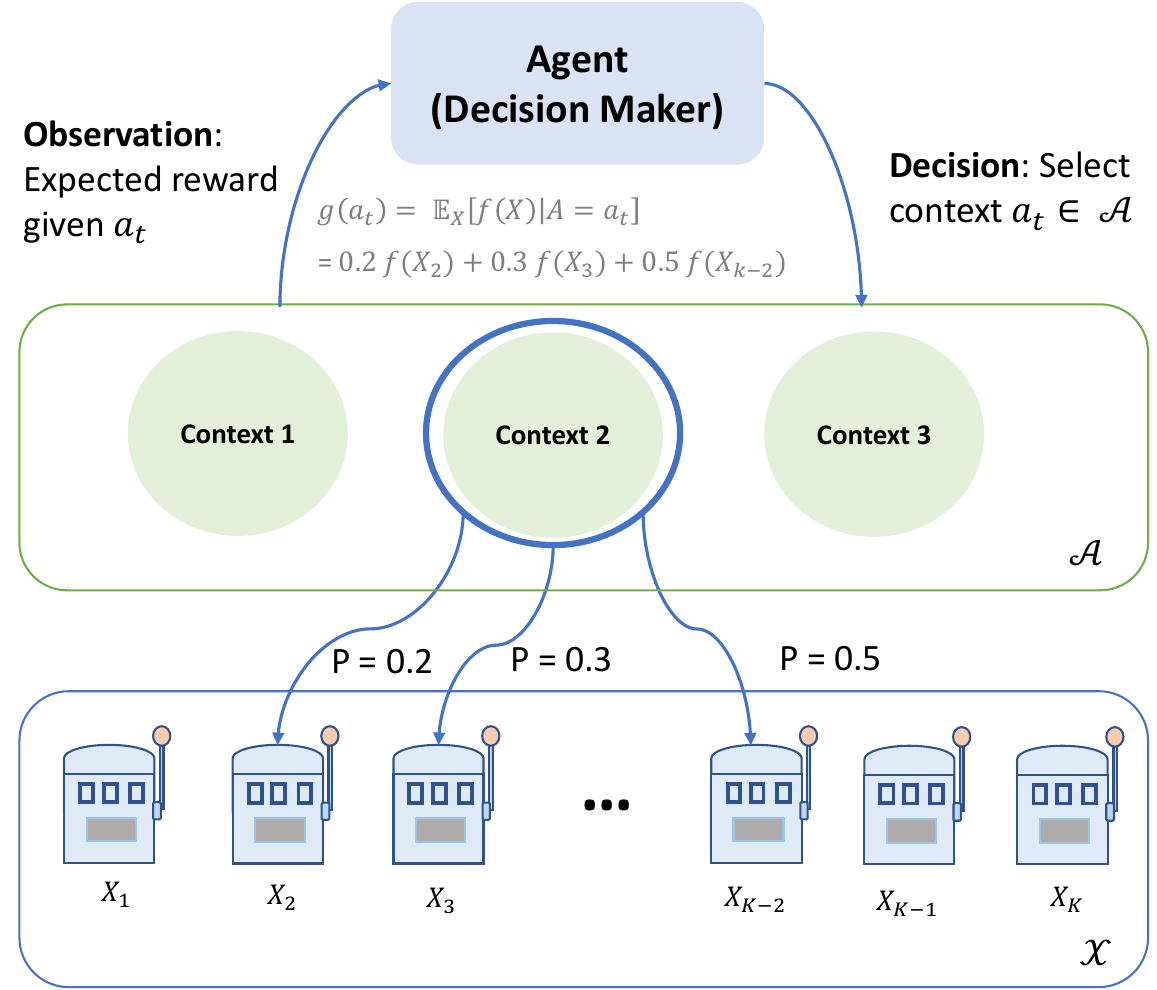}
    \caption{Illustration of indirect query bandits. The agent selects a context $a_t \in \mathcal{A}$ and observes the expectation feedback conditional on the choice of context. 
    }
    \label{fig:indirect query bandits}
\end{figure}

We aim to optimize the unknown function $f: \mathcal{X} \rightarrow \mathbb{R}$, where  $\mathcal{X} \subset \mathbb{R}^d$. 
In the classical BO setting \citep{garnett_bayesopIQBOok_2023}, the agent will at time $t$ query $x_t \in \mathcal{X}$ and receive feedback $f(x_t)$ (potentially with noise).
Instead, in each iteration $t$, we consider the indirect query $a_t \in \mathcal{A}$ according to a policy $\pi$ and observe noisy integrated feedback 
$$z_t = g(a_t) + \epsilon_t,$$
where the function $g$ is defined as the conditional expectation of $f$, 
\begin{align}
\label{equ: conditional expectation}
    g(a_t) := \mathbb{E}_X[f(X)|A = a_t] = \int_{\mathcal{X}} f(x) p(x|a_t) dx,
\end{align}
and $\epsilon_t \sim \mathcal{N}(0, \sigma^2)$.
The conditional distribution $p(x|a)$ can either be known or learned from given offline collected data $\{(x_j, a_j)\}_{j = 1}^N$, where $x_j$ is sampled from $p(X|A = a_j)$ and $a_j$ is sampled from $p(A)$. Note that the sequential query $a_t \in \mathcal{A}$ might not be present in the offline collected data. 
We denote the available data consisting of $T$ queries collated with the offline data as $\mathcal{D}^T = [\{(a_t, z_t)\}_{t=1}^T, \{(x_j, a_j)\}_{j = 1}^N]$. 

The integral feedback via conditional expectation provides a rich structure for the Bayesian optimization problem which can be useful in many applications. We provide two examples below to motivates our setting.

\paragraph{Motivation: Indirect Query Bandits}
We can view $a_t \in \mathcal{A}$ as contexts we can select, and this can determine the distribution of direct factors that we care about. 
We illustrate this in Figure \ref{fig:indirect query bandits} with a discrete space as an example. This can be viewed as a variation of the multi-armed bandit problem where we do not directly interact with the arms to get explicit rewards. 
Instead, we choose a context $a_t \in \mathcal{A}$ assumed to have a particular distribution over $\mathcal{X}$. By selecting a context, we will observe the average reward of that context. 

This setting may arise in online advertising, where only aggregated performance data (click-through or conversion rates) from multiple third-party platforms may exist \citep{avadhanulaStochasticBanditsMultiplatform2021}. 
Another application is in medical trials, e.g. dose finding trials \citep{azizMultiArmedBanditDesigns2021}, where agents may correspond to healthcare providers who assign treatment to patients based on a particular policy or protocol, and we observe the average recovery rates for each healthcare provider (outcomes for individual patients are not available due to privacy restrictions).
Function $g$ represents the aggregated feedback in such two-layer structure systems with a hidden mechanism $p(x|a)$ that chooses queries in $x$-space for a given context/healthcare provider.


\begin{figure}[t!]
    \centering
    \includegraphics[width=1\linewidth]{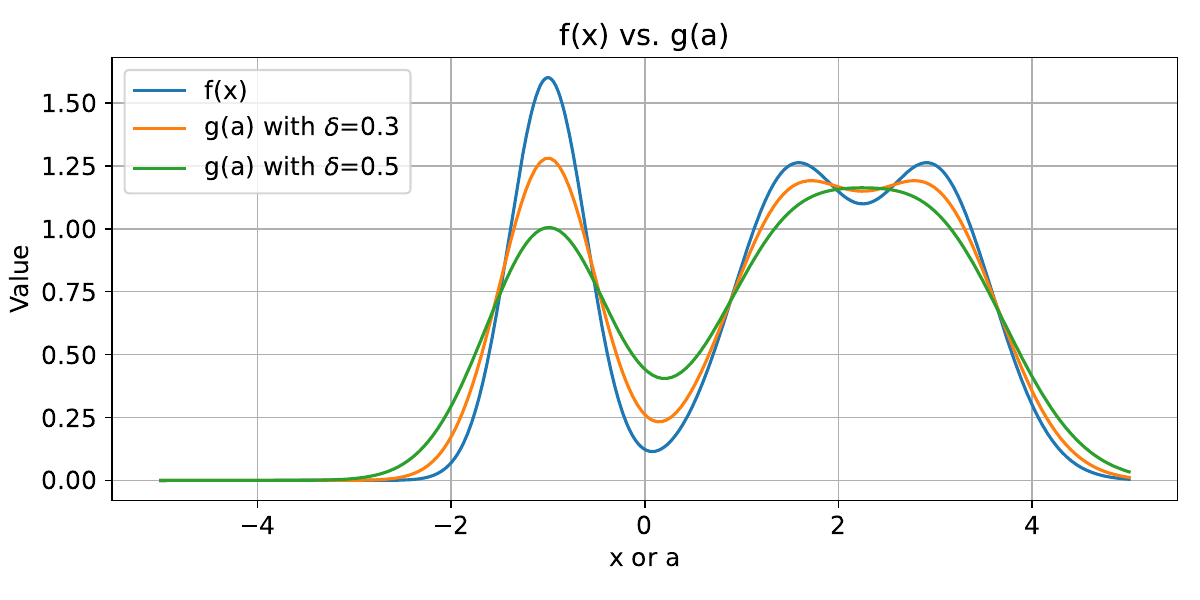}
    \caption{Illustration of multi-resolution example. Here, the query space is given by $a=(\xi,\delta)$ and $g(a)$ is a convolution of $f(x)$ with a Gaussian p.d.f. $\phi(x;\xi,\delta^2)$. Thus, $g$ corresponding to different resolutions $\delta$ approximates the true function $f$ using different precision. When $\delta = 0.5$, $f$ and $g$ have different optima.}
    \label{fig:multi-resolution}
\end{figure}

\paragraph{Motivation: Multi-resolution Queries} 
\label{sec: motivation Multi-fidelity} 
The IQBO framework covers the special case of multi-resolution queries if we expand the action space $\mathcal{A}$ to indicate the observation resolution level.
Let $a_t=(\xi_t,\delta_t)$, so that $X|A=a_t$ is a distribution with location $\xi_t$ and with a scale/resolution parameter $\delta_t$. 
That is 
$$\mathbb{E}[X|A = a_t] = \xi_t; \quad Var[X|A = a_t] = c \delta_t^2 I,$$
where $c$ is a known constant.
For example, when $p(x|a)$ is a 1D Gaussian distribution $\mathcal N(x;\xi,\delta^2)$, the noisy observation 
$z_t=\int_{-\infty}^{+\infty} f(x)\mathcal \phi(x;\xi_t,\delta_t^2)dx+\epsilon_t,$
where $\phi$ denotes the p.d.f. of Gaussian distribution;
when $p(x|a)$ is then uniform distribution on $(\xi-\delta,\xi+\delta)$, we observe $z_t=\frac{1}{2\delta_t}\int_{\xi_t-\delta_t}^{\xi_t+\delta_t} f(x)dx+\epsilon_t,$
with $\epsilon_t\sim \mathcal{N}\left(0,\sigma^{2}\right)$. 

At every iteration, we can control both the location and the resolution of the query of the function to be optimized, and observe an averaging approximation over the function.
As $\delta_t\to 0$, the integrated feedback reflects the precise function value on $f$, that is $g((\xi_t, \delta_t)) \to f(\xi_t)$.
We illustrate the multi-resolution example in Figure \ref{fig:multi-resolution}.

It is natural to associate a cost $\lambda(\delta)$ to the resolution $\delta$ such that smaller $\delta$ corresponds to a more precise observation and carries a higher cost. IQBO then becomes an attractive option for reducing the cost of queries, where we will seek to design policies where queries lead to the most information \emph{per unit cost}. 
The multi-resolution setup is close to the multi-fidelity Bayesian optimisation (MFBO) problem, which has been studied widely to learn functions via cheaper surrogates with varying cost and fidelity \citep{kandasamy2016gaussian, kandasamy2017multifidelitybayesianoptimisationcontinuous,wu2019practicalmultifidelitybayesianoptimization}. 
Our setup additionally considers indirect queries and transformation via conditional expectations, which is not covered in the standard MFBO problem. 



Applications include cloud satellite imagery, similar to the setup described in \citet{chau_deconditional_2021}. Satellite imagery often varies in resolution and is collected over different locations and times. Our objective is to find a location with the highest value of an atmospheric variable, such as cloud top temperature, over a high-resolution spatial domain. 
However, different level of resolutions would associate with different cost to collect the observations. 
Similarly, in malaria incidence models \citep{law_variational_2018}, we may be interested to find the highest underlying malaria incidence rate ($f(x)$) in each 1km by 1km region (pixel, $x$), while the available observations are aggregated incidence rate ($z$) of administrative units (bags of pixels, $a$). 



\paragraph{Evaluation Metric}
Our goal is to optimise the function $f$ given the indirect feedback over the observations from the function $g$. 
Define the optimal point of function $f$ as $x_\ast := {\arg\max}_{x \in \mathcal{X}} f(x)$. 
We evaluate the policy by regrets, which reflect how much our selections differ from $f(x_\ast)$. 
We first define the \textit{instant regret} in Definition \ref{defi: instant regret} to show the quality of our selections in query space $\mathcal{A}$. 

\begin{definition}[Instant Regret]
\label{defi: instant regret}
   Define instant regret at iteration $T$ as
\[
  r^g_{T}:=f(x_{*})- \max_{t \in [1,T]} g(a_{t}).
\] 
\end{definition}

In many applications, we are interested in identifying the optimal points in the target domain $\mathcal{X}$. 
Similar to the fixed budget best arm identification task in multi-armed bandits or BO \cite{audibert_best_2010}, the agent additionally recommends $x_T \in \mathcal{X}$ as a putative optimum at the iteration (budget) $T$, according to its policy $\pi^\prime$. 
We then define the simple regret in Definition \ref{equ: simple regret}, which evaluates the ability to find optimal points in $\mathcal{X}$ space after allocating the budget to querying the indirect space $\mathcal{A}$ and receiving the corresponding integrated feedback.


\begin{definition}[Simple regret]
\label{equ: simple regret}
    The simple regret at iteration $T$ is defined as
\begin{align*}
    r_T := f(x_\ast) - f(x_T).
\end{align*}
\end{definition}

We now discuss the assumptions we make in order to develop our algorithm and perform our analysis.
In Assumption \ref{ass: gp}, we model $f$ as a sample from a Gaussian Process (GP) to possess a degree of smoothness, which also enables us to deal with the integrated feedback and estimate $f$ in a Bayesian framework. 
We also make assumption for the conditional distribution to allow us to learn enough information around the optimal point of the function $f$ by querying $a$ indirectly in Assumption \ref{ass: p(x|a)}. 

\begin{assumption}
\label{ass: gp}
The function $f$ is a sample from a zero-mean Gaussian Process with known covariance function $k$.
\end{assumption}

\begin{assumption}
\label{ass: p(x|a)}
For any $\epsilon \geq 0$ and $\zeta \geq 0$, there exists a query $a \in \mathcal{A}$ such that the conditional distribution satisfies
\[
p\left(x \in [x_\ast - \epsilon,\ x_\ast + \epsilon] \,\middle|\, a \right) \geq 1 - \zeta.
\]
\end{assumption}







\section{Background and Related Work}
\label{sec: related work}


\paragraph{Gaussian Process and Conditional Mean Process}
A Gaussian process (GP) \( \{f(x)\}_{x \in \mathcal{X}} \) is a collection of random variables indexed by elements \( x \) in the input space \( \mathcal{X} \), such that any finite subset \( \{f(x_i)\}_{i=1,\dots,m} \) follows a multivariate Gaussian distribution \citep{rasmussen_gaussian_2006}.  
A GP is fully specified by its mean and covariance functions:
\begin{align}
   m(x) &= \mathbb{E}_f[f(x)], \\
k(x, x') &= \mathbb{E}_f[(f(x) - m(x))(f(x') - m(x'))]. 
\end{align}
Let $\mathcal{H}_k$ and $\mathcal{H}_\ell$ denote the reproducing kernel Hilbert spaces (RKHSs) over the input space $\mathcal{X}$ and the query space $\mathcal{A}$, with associated positive definite kernels $k(x, x')$ and $\ell(a, a')$, respectively.  
In the multi-resolution setting, the kernel on queries $a = (\xi, \delta) \in \mathcal{A}$ can be constructed as a product kernel that captures similarity in both location and resolution:
$\ell(a, a') = \ell_\Xi(\xi, \xi') \cdot \ell_\Delta(\delta, \delta').$
Since $f \sim \mathcal{GP}(m, k)$ and $g(a) = \mathbb{E}[f(X) \mid A = a]$ is defined as a linear functional of $f$, $g$ is itself a Gaussian process almost surely: $g \sim \mathcal{GP}(\nu, q)$,  
with mean and covariance functions:
\begin{align}
\label{equ: posterior of g, mean}
    \nu(a) &= \mathbb{E}_X[m(X) \mid A = a], \\
\label{equ: posterior of g, cov}
    q(a, a') &= \mathbb{E}_{X, X'}[k(X, X') \mid A = a, A' = a'],
\end{align}
for all $a, a' \in \mathcal{A}$, with $(X, A), (X', A') \sim p(x, a)$.

At iteration $t$, given data $\mathcal{D}^{t-1}$, we follow the Conditional Mean Process (CMP) framework \citep{chau_deconditional_2021} to infer the posterior mean $\nu_{t-1}(a)$ and posterior covariance $q_{t-1}(a, a')$ of $g$, as well as the posterior mean $m_{t-1}(x)$ and covariance $k_{t-1}(x, x')$ of $f$.  
Due to space constraints, we defer the full description of CMP inference and posterior computation to Appendix~\ref{appendix: CMP}.

\paragraph{BO policies} 
Several types of BO methods are popular in the literature. 
One type is Upper Confidence Bound (UCB) based policies, where one queries the data point with the highest upper confidence bound, which is constructed by the posterior mean plus weighted posterior standard deviation. Gaussian Process Upper Confidence Bound (GPUCB) comes with sound theoretical regret bounds \citep{srinivas_gpucb:information-theoretic_2012}, but the balance weights may be hard to adjust in practice. 
Entropy search policies \citep{hennig2012entropy_search} select the queries such that corresponding observations provide the highest mutual information or information gain with respect to the optimum. 
\citet{wang2017MES} further improved the computational efficiency by proposing the Max-Value Entropy Search (MES), which studied mutual information w.r.t. the target function value at the optimum, which is usually in a much lower dimension than the input space. 
MES was also extended in a multi-fidelity setting \citep{takeno2020MFMES}, inspired by which we proposed a hierarchical tree search algorithm with multi-resolution. 
We refer to \cite{garnett_bayesopIQBOok_2023} for a comprehensive review of BO methods.

\paragraph{Indirect Query and Integrated feedback}
Aggregated feedback was studied by \cite{zhang2022gaussian} under the Gaussian process bandits setting.
Their aggregated feedback is averaging over local continuous non-overleap areas, while we consider a general reward given by conditional expectation with unknown $p(x|a)$. 
They aimed to recommend a local area in $\mathcal{X}$ which gives the highest possible average score, while our goal is to find an optimizer for $f$. 
Linear partial monitoring \citep{kirschner2023linearpartialmonitoringsequential} provided a general framework to link actions and feedback via linear observation maps. Different from the linear bandits model, the reward and observation features are decoupled. If the conditional distributions are known in our setting, then the integrated feedback can be expressed as one special case of linear partial monitoring. 
Apart from that, transductive linear bandits \citep{fiez2019sequentialexperimentaldesigntransductive} considered the measurement and actions are separated, where they assumed the two spaces shared the same unknown parameters. 
In stochastic optimisation (chapter 11.9 in \cite{garnett_bayesopIQBOok_2023}), a line of work studied the opposite setting of ours, where direct feedback was observed but to optimise the aggregated target, either as in linear functionals \citep{mutný2023experimentaldesignlinearfunctionals, garciabarcos2020robustpolicysearchrobot}, with uncertainty inputs \citep{oliveira2019bayesianoptimisationuncertaininputs} or in conditional mean embedding setting \citep{chowdhury2020active}. 


\section{Methods}

In this section, we first propose a new policy called Conditional Max-value Entropy Search (CMES) to address our new setting. 
Then we propose a hierarchical search method in Algorithm \ref{alg: CMETS} to improve the computational efficiency in the multi-resolution setting. 
We focus on the design of policy $\pi$ to recommend $a_t \in \mathcal{A}$ in this paper. 
For policy $\pi^\prime$, where no direct queries are available, we simply recommend $x$ with the maximum posterior mean of $f$, i.e.
$x_t := {\arg\max}_{x \in \mathcal{X}} m_{t-1}(x)$. Note that in the IQBO setting, it is the policy $\pi$ which balances exploration and exploitation, while $\pi'$ simply posits the optimal input $x$ given the current knowledge (pure exploitation).

\begin{algorithm}[t!]
\caption{Conditional Max-value Entropy Search}
\label{alg: CMES}
    \begin{algorithmic}
        \STATE \textbf{Input:} dimension $d$, total iteration $T$, dataset $\mathcal{D}^0 = [\{(x_j, a_j)\}_{j = 1}^N, \{\}]$, query space $\mathcal{A}$.
        \FOR{$t = 1$ to $T$}
            \STATE Select $a_{t} = {\arg\max}_{a \in \mathcal{A}} \mathbb{I}(z; f_\ast| a, \mathcal{D}^{t-1})$ via Eq. (\ref{equ: CMES approximation}).
            \STATE Observe feedback $z_t = g(a_t) + \epsilon_t$.
            \STATE Update dataset $\mathcal{D}^t = \mathcal{D}^{t-1} \cup \{(a_t, z_t)\}$.
            \STATE Update posteriors of $f$ and $g$ given $\mathcal{D}^t$.
        \ENDFOR
        \STATE \textbf{Return} $x_{rec} = \arg\max_{x \in \mathcal{X}} m_{t-1}(x)$ 
    \end{algorithmic}
\end{algorithm}

\subsection{Conditional Max-value Entropy Search}
\label{sec: CMES}
Standard BO policies, such as GPUCB, PI or EI, optimize an expression which involves the posterior mean and variance of the target function $f$ to obtain the next query in $\mathcal X$. This is not applicable to IQBO since the queries reside in a different domain $\mathcal A$. Nor is it appropriate to apply these policies to the posterior of $g$, as that is not the function of interest. 
Thus, we turn our attention to the entropy-search type of policies. 

We develop a new policy, called Conditional Max-value Entropy Search (CMES) applicable to IQBO, which follows a strategy similar to Max-value Entropy Search (MES) \citep{wang2017MES}. At iteration $t$, CMES maximizes the mutual information between the noisy observation $z$ and $f_\ast = f(x_\ast)$ conditionally on the query $a$ and on the available datas $\mathcal{D}^{t-1}$ as defined in Eq. (\ref{equ: DMES}). In other words, CMES queries $a \in \mathcal{A}$ that provides the maximum information gain about the $f_\ast$, 
\begin{align}
 \label{equ: DMES}
    a_{t} := {\arg\max}_{a \in \mathcal{A}} \mathbb{I}(z; f_\ast| a, \mathcal{D}^{t-1}). 
\end{align}
The mutual information $\mathbb{I}(z;f_\ast| a, \mathcal{D}^{t-1})$ can be calculated as
\begin{align}
\label{equ: DMES cal}
    \mathrm{H}[z|a, \mathcal{D}^{t-1}] - \mathbb{E}_{f_\ast| \mathcal{D}^{t-1}}[\mathrm{H}[z|f_\ast, a, \mathcal{D}^{t-1}]],
\end{align}
with differential entropy given by $\mathrm{H}[u] :=-\int p(u) \log p(u) d x$.

\paragraph{CMES Approximation}
To compute the first term, we note that $z=\int f(x)p(x| a)dx+\epsilon$ has a normal distribution conditionally on $a$ so its differential entropy can be directly calculated from its posterior predictive variance as $$\mathrm{H}[z|a,\mathcal{D}^{t-1}] = 0.5 {\log}[2 \pi e (q_{t-1}(a,a) + \sigma^2)],$$
where $q_{t-1}$ is the posterior covariance function of $g$ given data $\mathcal{D}^{t-1}$, and $\sigma^2$ is the variance of noise term $\epsilon$.
For the second term, the expectation can be approximated using Monte Carlo estimation, i.e. 
\begin{align*}
    \mathbb{E}_{f_\ast| \mathcal{D}^{t-1}}[\mathrm{H}[z|f_\ast, a, \mathcal{D}^{t-1}]] \approx \frac{1}{|\mathcal{F}^t_\ast|}\sum_{f^t_\ast \in \mathcal{F}^t_\ast}  \mathrm{H}[z|f^t_\ast, a, \mathcal{D}^{t-1}],
\end{align*}
where $\mathcal{F}^t_\ast$ is the set of optimal value samples, which can be collected via the Thompson sampling with the maximum of samples from the posterior at each iteration, or via the Gumbel sampling introduced in \cite{wang2017MES}.
We use truncated normal distribution $p(z|z \leq f_\ast, a,\mathcal{D}^{t-1})$ to approximate $p(z|f_\ast, a,\mathcal{D}^{t-1})$ to allow analytically calculating $\mathrm{H}[z|f_\ast, a,\mathcal{D}^{t-1}]$. 
We further discuss the choice of using $p(z|g(a) \leq f_\ast, \mathcal{D}^{t-1})$ to approximate $p(z|f_\ast, a, \mathcal{D}^{t-1})$ in the Appendix \ref{sec: CMES appendix}. 
Define the \textit{standardized improvement margin (SIM)} w.r.t. $f_*$ as
\begin{align}
\label{equ: gamma}
    \gamma_{f_*}(a) := \frac{f_* - \nu_{t-1}(a)}{\sqrt{q_{t-1}(a, a)}}.
\end{align}
Further define the function 
$$h(\alpha):= \frac{\alpha \phi(\alpha)}{2\Phi(\alpha)} - \log \Phi(\alpha),$$ 
where $\phi, \Phi$ are the probability density function and the cumulative density function of a standard Gaussian distribution. 
Then the entropy term in the second term in Eq. (\ref{equ: DMES cal}) can be approximated by 
$$
   \mathrm{H}[z|f_\ast, a, \mathcal{D}^{t-1}] \approx 0.5 \log \left(2 \pi e (q_{t-1}(a,a) + \sigma^2)\right) - h(\gamma_{f_*}(a)).
$$
Finally we can approximate Eq. (\ref{equ: DMES cal}) by
\begin{align}
\label{equ: CMES approximation}
    \mathbb{I}(z;f_\ast| a, \mathcal{D}^{t-1}) \approx \frac{1}{|\mathcal{F}^t_\ast|} \sum_{f_*^t \in \mathcal{F}^t_\ast} h(\gamma_{f^t_*}(a)).
\end{align}

The CMES policy picks the $a \in \mathcal{A}$ that maximises the mutual information as defined in Eq. (\ref{equ: DMES cal}), and approximated in Eq. (\ref{equ: CMES approximation}). 
Since $h$ is a strictly decreasing function, maximising Eq. (\ref{equ: CMES approximation}) seeks small SIM terms $\gamma_{f_*^t}(a)$.
This term becomes small when either the posterior mean $\nu_{t-1}(a)$ is large, favouring exploitation of the current estimated maximum; or the posterior standard deviation $\sqrt{q_{t-1}(a, a)}$ is large, encouraging exploration of uncertain regions.







\input{CMETS}
\section{Theoretical analysis}
\label{sec: theory}

In this section, we derive the regret upper bounds for IQBO in a variety of settings, assuming the conditional distribution $p(x|a)$ is known. The full proofs of all theoretical results are given in the Appendix \ref{sec: proof appendix}.





\subsection{Instant Regret Bound of CMES Policy}

In this section, we show the upper bound of instant regret (Definition \ref{defi: instant regret}) for our proposed policy CMES. 
We start by showing the equivalence among CMES and other baseline policies (introduced in Section \ref{sec: related work}) in Lemma \ref{lemma: equivalence among policies}. 
This is an important link for proof of regret bound, as we will follow the proof technique in \cite{wang2017MES, pmlr-v51-wang16f}, which relies on the equivalence to GPUCB \citep{srinivas_gpucb:information-theoretic_2012}. 

\begin{lemma} 
\label{lemma: equivalence among policies}
We show that our CMES (with one sample $f^t_\ast$ of Eq. (\ref{equ: CMES approximation}) approximation) is equivalent to:
\begin{itemize}
    \item MES w.r.t function $g$ and use $f^t_\ast$ as the optimal value sample (instead of $g_\ast$).
    \item EST w.r.t. function $g$ and $m = f^t_\ast$, where $m$ denotes the the function optimal in EST.
    \item GPUCB w.r.t. function $g$ with $\beta_t^{1/2} = \min_{a \in \mathcal{A}} \frac{f^t_* - \nu_{t-1}(a)}{\sqrt{q_{t-1}(a,a)}}$, where $\beta_t$ is the coefficient of the posterior standard deviation.
\end{itemize}
\end{lemma}


Lemma \ref{lemma: equivalence among policies} shows that the trick to deal with the indirect query in mismatch space is to maintain the posteriors on both function $f$ and $g$, where we use posterior mean and standard deviation of $g$ for query, but the optimal value samples of $f$ via the posterior of $f$ for plug-in parameters.

In Lemma \ref{lemma: concentration g EST}, we show the concentration inequality w.r.t. function $g$, corresponding to $g(a)|\mathcal{D}^{t-1} \sim \mathcal{N}(\nu_{t-1}(a), q_{t-1}(a,a))$.

\begin{lemma}[Concentration inequality of $g$]
\label{lemma: concentration g EST}
    Pick $\delta \in(0,1)$ and set $\zeta_t=\left(2 \log \left(\frac{\pi_t}{2 \delta}\right)\right)^{\frac{1}{2}}$, where $\sum_{t=1}^T \pi_t^{-1} \leq$ $1, \pi_t>0$. Then, it holds that $P\left[|\nu_{t-1}\left(a\right)-g\left(a\right)| \leq\right.$ $\left.\zeta_t \sqrt{q_{t-1}(a,a)}, \forall t \in[1, T]\right] \geq 1-\delta$.
\end{lemma}

With the concentration inequality in Lemma \ref{lemma: concentration g EST}, we upper bound the regret that occurs in time step $t$ w.r.t. function $f$ with posterior $q_{t-1}(a_t,a_t)$ in Lemma \ref{lemma: EST regret bound}.
We further show the links between the posterior $q_{t-1}(a_t,a_t)$ and information gain $I(\bm{z}_T; f)$ in Lemma \ref{lemma: information gain - var}. 

\begin{lemma}
\label{lemma: EST regret bound}
    With $\kappa_t = \min _{a \in \mathcal{A}} \frac{\hat{l}_t-\nu_{t-1}(a)}{\sqrt{q_{t-1}(a,a)}}$, where $\hat{l}_t \geq \max _{x \in \mathcal{X}} f(x)$, $\forall t \in[1, T]$, the regret $\tilde{r}^g_t$ at time step $t$ is upper bounded as 
    $$\tilde{r}^g_t := \max_{x \in \mathcal{X}} f(x) - g(a_t) \leq\left(\kappa_t+\zeta_t\right) \sqrt{q_{t-1}(a_t,a_t)}.$$
\end{lemma} 

\begin{remark}
To upper bound simple regret w.r.t $f(x_t)$, we need to bound the difference between posterior means of $a_t$ (w.r.t $g$) and $x_t$ (w.r.t. $f$), i.e. $r_t \leq \nu_{t-1}(a_t) + \kappa_t \sqrt{q_{t-1}(a_t,a_t)} + \zeta_t \sigma_{t-1}\left(x_t\right) - m_{t-1}\left(x_t\right)$. 
This would require additional assumptions between $a_t$ and $x_t$ (and assumptions for conditional distributions). 
\end{remark}


\begin{lemma}
\label{lemma: information gain - var}
    The information gain for the selected points can be expressed in terms of the predictive variances. If $\bm{z}_T = (z(a_t)) \in \mathbb{R}^T$,
    \begin{align*}
        \mathbb{I}(\bm{z}_T; f) = \frac{1}{2} \sum_{t=1}^T \log (1 + \sigma^{-2} q_{t-1}(a_t, a_t)).
    \end{align*}
\end{lemma}



With Lemma \ref{lemma: EST regret bound} and \ref{lemma: information gain - var}, we now show the instant regret upper bound in Theorem \ref{theo: Instant Regret Bound CMES}.
Let $F$ be the cumulative probability distribution for the maximum of any function $f$ sampled from $GP(m, k)$ over the compact space $\mathcal{X} \subset \mathbb{R}^d$, where $k\left(x, x^{\prime}\right) \leq 1, \forall x, x^{\prime} \in \mathcal{X}$. 
Let $w:=F\left(f_*\right) \in(0,1)$, and assume the observation noise is iid $\mathcal{N}(0, \sigma^2)$.
Define the maximum information gain between $\bm{z}$ and $f$ with observations of size $T$ as
\begin{align}
\label{equ: max information gain}
    \rho_T := \sup_{A \subset \mathcal{A}: |A| = T} \mathbb{I}(\bm{z}_A; f).
\end{align}

\begin{theorem}[Instant Regret Bound for CMES]
\label{theo: Instant Regret Bound CMES}
If $a_t$ is selected via CMES (with one sample $f^t_\ast$ of Eq. (\ref{equ: CMES approximation}) approximation),
then with probability at least $1-\delta$, in $T^{\prime}=\sum_{i=1}^T \log _w \frac{\delta}{2 \pi_i}$ number of iterations, the instant regret satisfies
$$
r^g_{T^{\prime}} \leq \sqrt{\frac{C \rho_T}{T}}\left(\kappa_{t^*}+\zeta_T\right),
$$
where $C=2 / \log \left(1+\sigma^{-2}\right)$ and $\zeta_T=\left(2 \log \left(\frac{\pi_T}{\delta}\right)\right)^{\frac{1}{2}} ; \pi_t$ satisfies $\sum_{i=1}^T \pi_i^{-1} \leq 1$ and $\pi_t>0$, and $t^*=\arg \max _t \kappa_t$ with $\kappa_t = \min _{x \in \mathcal{A}, f_*^t>f_*} \gamma_{f_*^t}(a)$. 
\end{theorem}


\subsection{Regret of Multi-resolution Queries}
\label{sec:budgeted-iqbo}
We now analyze the regret in the setting where the cost of queries depends on their resolution, as motivated in Section \ref{sec: motivation Multi-fidelity} and Algorithm \ref{alg: CMETS}.
Recall that we write $a_t = (\xi_t, \delta_t)$, with $\xi_t$ the location of the query and $\delta_t$ the resolution, associated with the cost $\lambda(\delta_t)$. We wish to perform cheap queries at the beginning, and increase resolution (i.e. decrease $\delta_t$) as iterations proceed. Instead of optimizing $\delta_t$ at each iteration as it is a part of $a_t$, like in Section \ref{sec: motivation Multi-fidelity}, for the purposes of analysis, we will consider a deterministic schedule of decreasing $\delta_t$ and optimize only $\xi_t$. This setting is considerably simpler since we effectively operate on the same space, as $\xi_t\in\mathcal X$, but note that $f$ and $g$ may still have different optima (an example of this is given in Fig. \ref{fig:multi-resolution}). 
We now investigate the instant regret bounds when selecting $\xi_t$
using the state-of-art baselines such as GP-UCB \citep{srinivas_gpucb:information-theoretic_2012} with respect to the posterior of target function $f$ and show that one can match regret bounds of direct queries with a cost that grows at significantly slower rates.

We rewrite
$g(a_t) = \int f(\xi_t + \eta_t) p_t(\eta_t) d\eta_t$
where $\eta_t=x-\xi_t$ is a random variable for which, conditionally on $a_t$, $\mathbb{E}\left[\eta_t\right]=0$ and $\mathrm{Var}\left[\eta_t\right]=\delta_{t}^{2}I$. We will make an assumption on tail decay of the distribution of $\eta_t$ in the following theorem. We note that this assumption is satisfied if $\eta_t/\delta_t$ has a sub-Gaussian distribution.
Define the maximal information gain for \emph{direct queries} (i.e. $\delta_t=0$) as
$$
    \gamma_T := \sup_{S \subset \mathcal{X}: |S| = T} \mathbb{I}(\mathbf{f}_S+\boldsymbol{\epsilon}, \mathbf{f}_S)= \sup_{S \subset \mathcal{X}: |S| = T}  \frac{1}{2} \log|I+\sigma^{-2}K_S|
$$
where $K_S$ is the covariance matrix of $\mathbf{f}_S=[f(x)]_{x\in S}$ and $\boldsymbol{\epsilon}\sim \mathcal{N}(0,\sigma^2 I_T)$.





\begin{theorem}[Instant regret bound]
\label{theo: instant regret bound gpucb}
Let $\delta_t$ follow a deterministic schedule and let $\xi_{t}=\arg\max_{\xi\in \mathcal X} m_{t-1}(\xi)+ \sqrt{ \beta_{t} k_{t-1}(\xi,\xi)}$. Assume that the third moment of $\Vert \eta_t/\delta_t \Vert$ is uniformly bounded, i.e. there exists $C>0$ such that $\sup_{t\geq 0}\mathbb{E}\Vert \eta_t/\delta_t\Vert^3\leq C$. Further, assume that the true function $f$ has a bounded second derivative, i.e. $\mathrm{Tr}(\nabla^2 f(x))\leq 2M$. Then the instant regret (Definition \ref{defi: instant regret}) is upper bounded by $\mathcal{O} \left(\sqrt{\frac{\beta_T \gamma_T + \sum_{t=1}^{T}\delta_{t}^{4}}{T}}\right)$. Moreover, if we set $\delta_t^2=O(t^{-1/2}(\log t)^{d/2})$ and use a Gaussian kernel prior on $f$, the upper bound of instant regret is $\mathcal{O}(\sqrt{\beta_T (\log T)^{d+1} /T})$. 
\end{theorem}

This bound matches that of direct queries of $f$ \citep{srinivas_gpucb:information-theoretic_2012}. To get further intuition, assume that $\eta_t$ is uniform on $[-\tau_{t},\tau_{t}]^{d}$, so $\delta^{2}_t=\frac{\tau_{t}^{2}}{3}$. This means that our window size for the queries can decrease as slowly as $\tau_{t}=O(t^{-1/4}(\log t)^{d/4})$ without affecting the regret. If the cost of queries scales as the inverse of the window size, this can lead to large savings, as illustrated in the Appendix.

\section{Experiments}
\label{sec: Experiments}

We implement our experiments with Python 3.10 with GpyTorch and BoTorch libraries. 
We evaluated our algorithm by both instant regret and the simple regret defined in Definition \ref{defi: instant regret} and \ref{equ: simple regret}. 
We choose state-of-the-art Bayesian optimization algorithms, including Max-value Entropy Search (MES),  Upper Confidence Bound (UCB) and Expected Improvement (EI) as our baselines \citep{garnett_bayesopIQBOok_2023}, and modify them to fit into our new setting. 
To select $a_t \in \mathcal{A}$, we ignore $\{(x_i, a_i)\}_{i=1}^N$ and update the posteriors of function $g$ based on sequentially collected $\{(a_j, z_j)\}_{j=1}^t$. 
We test our algorithms on the synthetic optimization function Branin in 2D, and construct the conditional distribution as $\mathcal{N}(x| h(a), \delta^2)$ with $h$ to be linear and non-linear transformations and $a \in [0,1]^2$.
We specify the detailed setup in Appendix \ref{appendix: experimental settings}.
We consider two experimental settings.

Firstly, we test Algorithm \ref{alg: CMES} with fixed iterations and fixed resolution $\delta = 0.5$ and discretize the search space into grids. We show both the simple regret and instant for 100 iterations in Figure \ref{fig:iteration regret}. 
The confidence interval is shown with 10 independent experiment runs. 
Our proposed method CMES outperforms baselines for both regrets, especially in early iterations, since our approach directly optimizes the function $f$ and uses the data to learn conditional distributions. 

\begin{figure}
    \centering
    \includegraphics[scale=0.32]{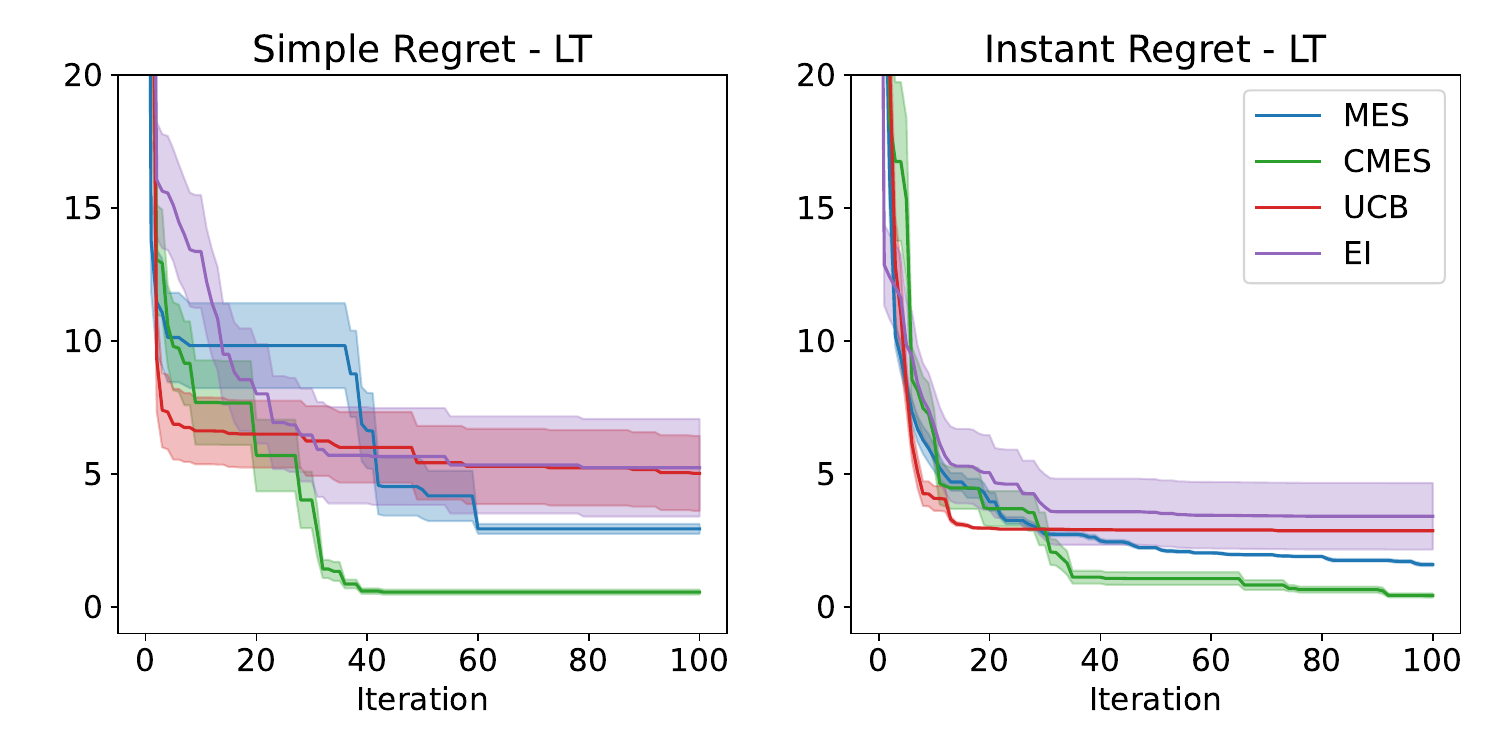}
    \includegraphics[scale=0.32]{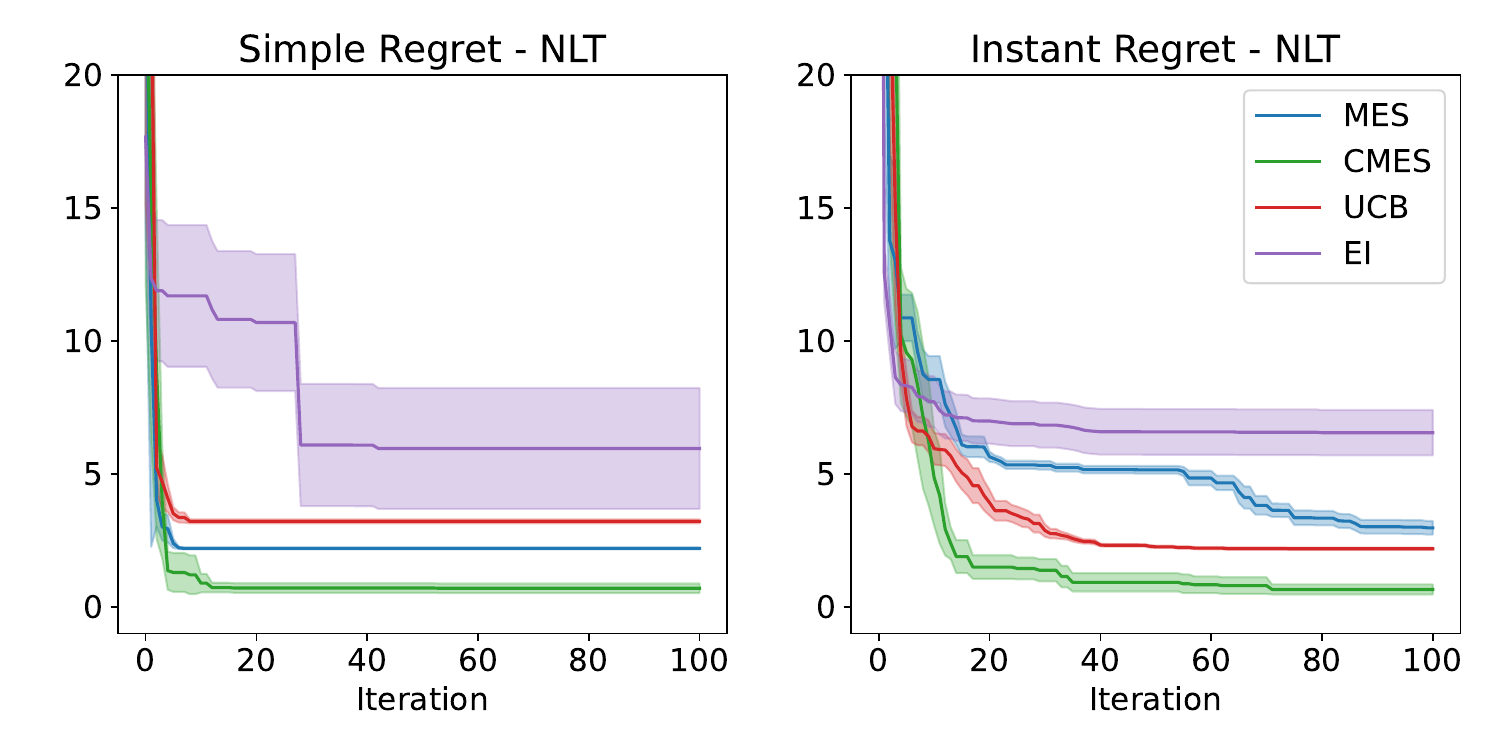}
    \caption{Regrets for Algorithm \ref{alg: CMES}. 
    LT: linear transformations, NLT: non-linear transformations. }
    \label{fig:iteration regret}
    \
\end{figure}

\begin{figure}[t!]
    \centering
    \includegraphics[scale=0.32]{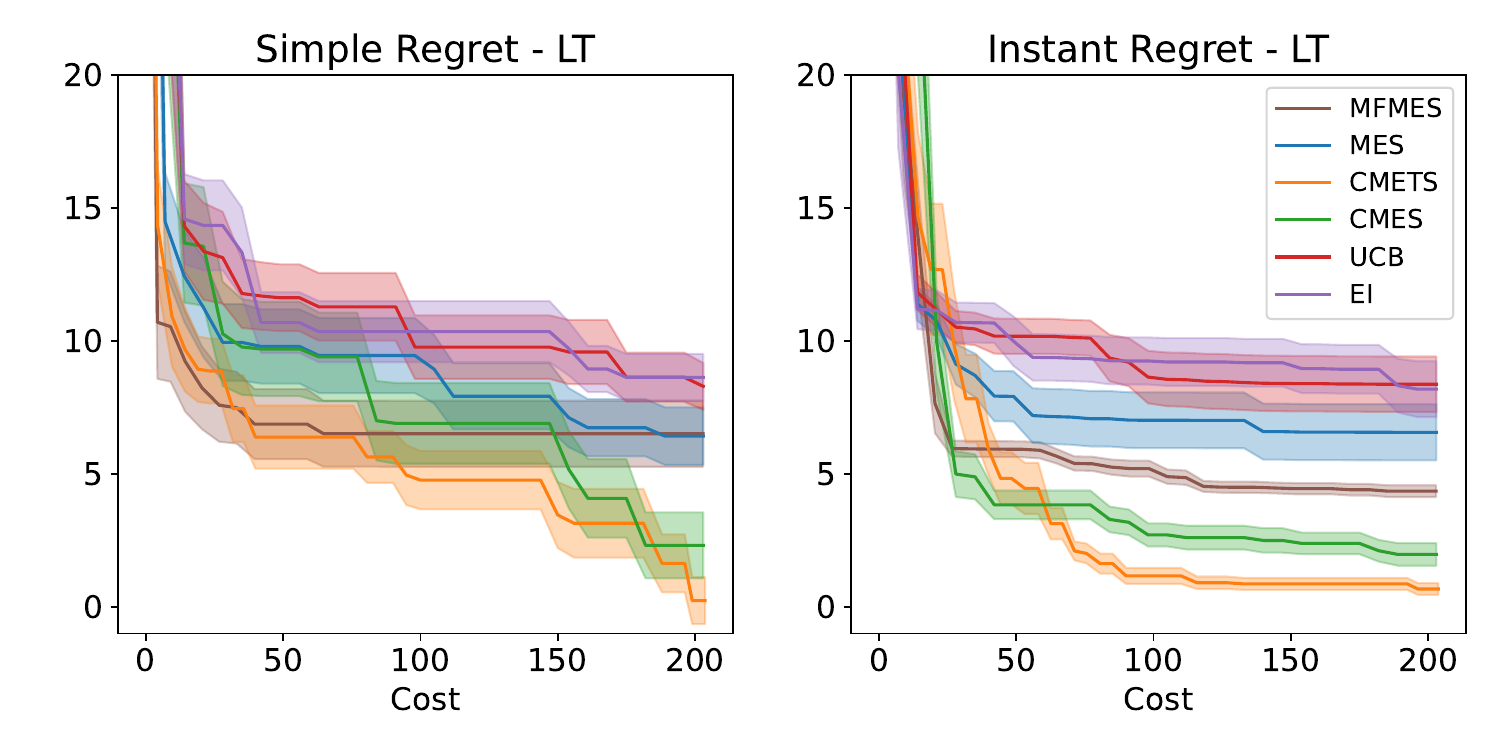}
    \includegraphics[scale=0.32]{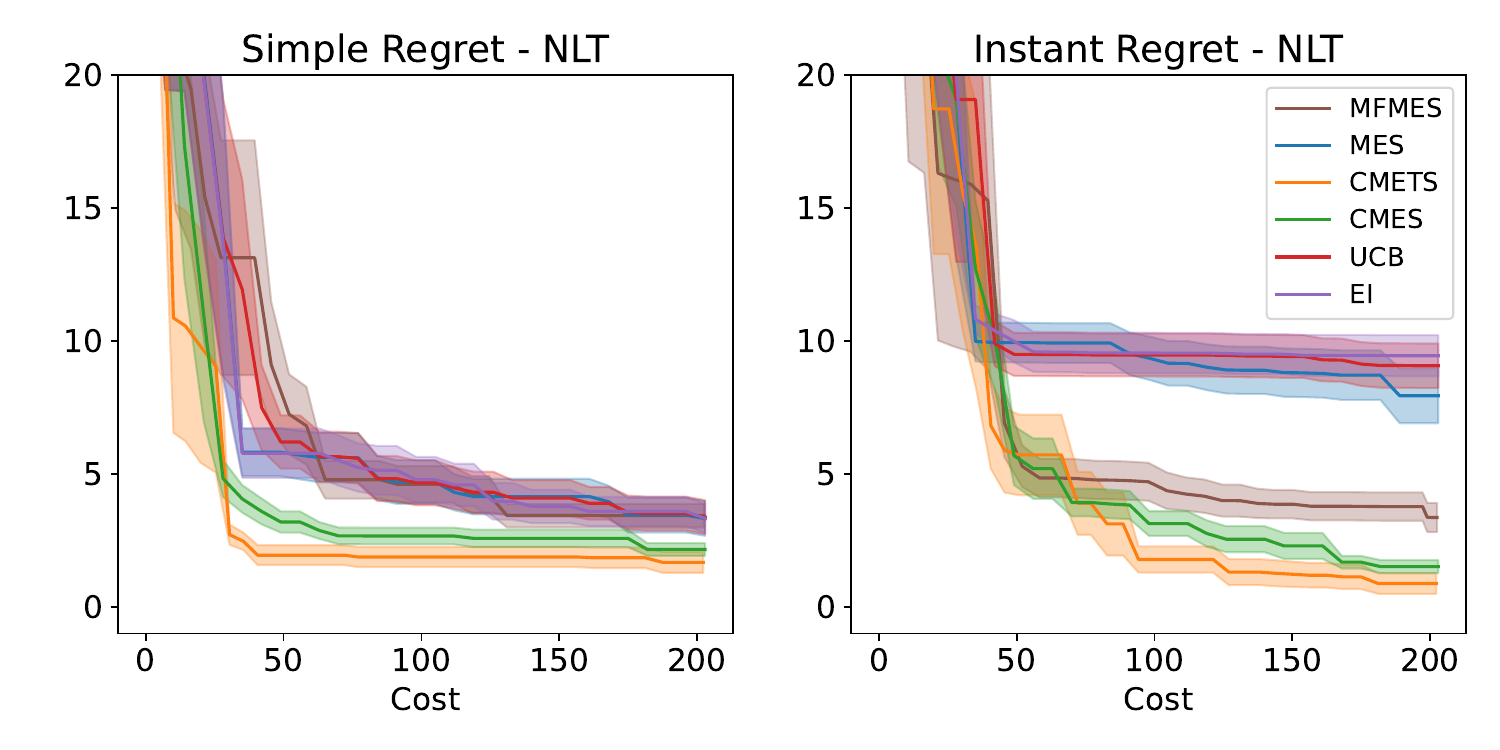}
    \caption{Regrets for Algorithm \ref{alg: CMETS}, budgeted based multi-resolution setting 
    (cost and resolution level associated with node height in the tree). 
    LT: linear transformations, NLT: non-linear transformations.
    }
    
    \label{fig:experiment}
\end{figure}

Secondly, we design experiments to test Algorithm \ref{alg: CMETS} for budgeted multi-resolution setting. 
We construct a binary tree in 2D space with pre-defined highest height $l_{max} = 6$. That is, we split each node into 4 non-overlap 2D regions with equal size. Each node is represented by its centre. The noise standard deviations and costs are set according to the height of nodes (in a logarithm rate). 
For non-hierarchical policies (CMES, MES, UCB, EI), we use all leaf nodes as the candidate sets. 
For hierarchical (multi-resolution) policies, the candidate set is updated as an active candidate set as specified in Algorithm \ref{alg: CMETS}.
We add an additional baseline model, the Multi-Fidelity MES (MFMES) \citep{takeno2020MFMES} and adapted it into the tree-level fidelity levels. To be consistent with other models and for fair comparisons, the fidelity level is associated with the tree level as shown in Algorithm \ref{alg: CMETS} and cannot be queried separately.  
For hierarchical search, we set the cost as $\lambda_a = 0.5 * {\log}_2(1/d(a))$, where the radius $d(a) =\frac{1}{2^{l(a)+1}}$ and $l(a)$ is the height level the node in the tree. And the resolution is set to $\delta = 0.5/\lambda_a$. 
We show regret results in Figure \ref{fig:experiment}, where the x-axis is the total cost budget, and the y-axis is the simple regret (left column) and instant regret (right column).
We can observe policies (CMETS, MFMES) that are resolution-aware and with hierarchical search outperform their corresponding base algorithms (CMES, MES). 
Out of all the policies, our proposed algorithms CMES and CMETS are the best ones since they explicitly consider integrated feedback.

\section{Conclusion}

We introduce a new setting for Bayesian Optimization, namely Indirect Query Bayesian Optimization (IQBO), where the feedback is integrated via a conditional distribution.
We motivate our new setting with indirect query bandits and multi-resolution queries. 
We propose the Conditional Max-Value Entropy Search (CMES) in Algorithm \ref{alg: CMES} to address this new setting.
We further propose 
Conditional Max-Value Entropy Tree Search (CMETS) in Algorithm \ref{alg: CMETS} for the multi-resolution setting and improve computational efficiency. 
Our approach is supported by theoretical analysis demonstrating the upper bound of regrets.  Empirical evaluations on synthetic optimization functions demonstrate the effectiveness of our approach in comparison with existing state-of-the-art methods.

\bibliography{ref}

\newpage

\clearpage
\appendix
\setcounter{secnumdepth}{1}  
\renewcommand{\thesection}{\Alph{section}}  
\makeatletter
\renewcommand{\@seccntformat}[1]{\csname the#1\endcsname\quad}  
\makeatother

\input{proof}


\end{document}

%% file: lettres_maths.tex
\def\11{{\mathbf 1}}    



%% file: CMETS.tex
\subsection{Conditional Max-value Entropy Tree Search 
}
\label{sec:hierarchical}

\begin{algorithm}[t!]
\caption{Conditional Max-value Entropy Tree Search}
\label{alg: CMETS}
    \begin{algorithmic}
        \STATE \textbf{Input:} dimension $d$, $K$-ary tree, cost $\lambda_l$ ($\lambda_l < \lambda_{l+1}$), total budget $B$, dataset $\mathcal{D}^0 = [\{(x_j, a_j)\}_{j = 1}^N, \{\}]$.
        \STATE \textbf{Initialise:} tree $\mathcal{T}_0 = \{(\xi_{root},\delta_{0})\}$; leaves $\mathcal{L}_0 = \mathcal{T}_0$; initial active candidate set $\mathcal{A}_0 = \mathcal{L}_0 \cup \mathcal{L}_0^\prime$, 
        where $\mathcal{L}_0^\prime$ is the set containing all children of each node in $\mathcal{L}_0$; 
        $t = 1$.
        \WHILE{$B > 0$}
            \STATE Select $a_{t} = {\arg\max}_{a \in \mathcal{A}_{t-1}} \mathbb{I}(z; f_\ast| a, \mathcal{D}^{t-1})/ \lambda_{l(a)}$
            \STATE Observe feedback $z_t = g(a_t) + \epsilon_{t}$.
            \STATE Update dataset $\mathcal{D}^t = \mathcal{D}^{t-1} + \{(a_t, z_t)\}$.
            \STATE Update posteriors of $f$ and $g$ given $\mathcal{D}^t$.
            \IF{$a_{t} \in \mathcal{L}_{t-1}^\prime$ }
                \STATE $\mathcal{L}_{t} = \mathcal{L}_{t-1} - \{\text{parent}(a_{t})\} + \{\text{children}(a_{t})\}$
                \STATE $\mathcal{L}_t^\prime = \mathcal{L}_t^\prime - \{a_t\}$
            \ELSE{}
                \STATE $\mathcal{L}_{t} = \mathcal{L}_{t-1} - \{a_{t}\} + \{\text{children}(a_{t})\}$
            \ENDIF
            \FOR{node $i \in \mathcal{L}_{t}$:}
                \STATE{$\mathcal{L}^\prime_t = \mathcal{L}^\prime_t$ $\cup \{\text{children}(i$)\}}
            \ENDFOR
            \STATE Update active candidate set $\mathcal{A}_t = \mathcal{L}_{t} \cup \mathcal{L}_{t}^\prime$
            \STATE Update budget $B = B- 
                \lambda_{l(a_t)}$, $t = t + 1$.
        \ENDWHILE
        \STATE \textbf{Return} $x_{rec} = \arg\max_{x \in \mathcal{X}} m_{t-1}(x)$ 
    \end{algorithmic}
\end{algorithm}

In this section, we focus on the multi-resolution approximations discussed in Section \ref{sec: motivation Multi-fidelity} and incorporate them through hierarchical partitioning to improve cost efficiency. 
We consider a pre-defined $K$-ary tree, the root represents the query space $\mathcal{A}$, and for each node, there are $K$ children partitioning the space their parent represents. 
Each node is denoted as $a = (\xi, \delta_{l(a)})$ as specified in Section \ref{sec: problem setup and cmp} multi-resolution setting, where the resolution level $\delta_{l(a)}$ depends on the height of the node $l(a)$. 
The higher the level $l(a)$ of the node is, the smaller $\delta_{l(a)}$ is, and the agent can observe a more precise approximation reward for the function $f$. 
Correspondingly, querying nodes in a higher level of the tree is associated with a higher cost $\lambda_{l(a)}$.
We assume the cost function is known. 

We propose the Conditional Max-value Entropy Tree Search in Algorithm \ref{alg: CMETS}, where we partition the continuous space adaptively following a tree structure with a limited resource budget $B$.
In each iteration $t$, we maintain an active candidate set $\mathcal{A}_t = \mathcal{L}_{t} \cup \mathcal{L}_{t}^\prime$,  where $\mathcal{L}_{t}$ is the set of all leaves at iteration $t$ and $\mathcal{L}_{t}^\prime$ is the set of children of each node in $\mathcal{L}_{t}$.
$\mathcal{L}_{0}$ is initialised as the root node.
In each iteration, an action $a_t = (\xi_t, \delta_{l(a_t)})$ is selected from the active candidate set $\mathcal{A}_{t-1}$, balancing the information gain and cost by weighting our proposed policy CMES with the inverse $1/ \lambda_{l(a)}$. 
Specifically, the selected action $a_t$ is determined as the one that maximizes the ratio of mutual information $\mathbb{I}(z, f_\ast | a, \mathcal{D}^{t-1})$ to the cost $\lambda_{l(a)}$ associated with its height.
The action selection criterion aims to optimize information gain per unit cost, ensuring efficient use of the budget. 
After selecting an action $a_t$, the corresponding node of the tree is then split to explore more precise regions and the active candidate set $\mathcal{A}_t$ is updated accordingly. 

%% file: proof.tex
\section*{Appendix}

\section{Posterior Inference via CMP}
\label{appendix: CMP}

In this section, we present the posterior updates derived using the Conditional Mean Process (CMP) framework \citep{chau_deconditional_2021}. Let $\mathcal{H}_k$ and $\mathcal{H}_\ell$ be the reproducing kernel Hilbert spaces (RKHSs) over $\mathcal{X}$ and $\mathcal{A}$ with kernels $k(x, x')$ and $\ell(a, a')$, respectively.

Define the canonical feature maps as $k_x := k(x, \cdot)$ and $\ell_a := \ell(a, \cdot)$ for $x \in \mathcal{X}$ and $a \in \mathcal{A}$. Let the feature matrices and Gram matrices be defined as:
\begin{align*}
    \Phi_{\mathbf{x}} &:= \begin{bmatrix} k_{x_1} & \cdots & k_{x_N} \end{bmatrix}, \quad \mathbf{K}_{\mathbf{xx}} := \Phi_{\mathbf{x}}^\top \Phi_{\mathbf{x}} = [k(x_i, x_j)]_{i,j=1}^N, \\
    \Phi_{\mathbf{a}} &:= \begin{bmatrix} \ell_{a_1} & \cdots & \ell_{a_N} \end{bmatrix}, \quad \mathbf{L}_{\mathbf{aa}} := \Phi_{\mathbf{a}}^\top \Phi_{\mathbf{a}} = [\ell(a_i, a_j)]_{i,j=1}^N.
\end{align*}

Kernel mean embeddings provide a powerful nonparametric framework for representing and reasoning about distributions in RKHSs \citep{muandet_kernel_2017}. The Conditional Mean Embedding (CME) \citep{song_hilbert_2009} embeds a conditional distribution $p(x \mid a)$ into $\mathcal{H}_k$ via
\[
\mu_{X \mid A = a} := \mathbb{E}_X[k_X \mid A = a] \in \mathcal{H}_k.
\]
Define the Conditional Mean Operator (CMO) $C_{X \mid A}: \mathcal{H}_\ell \to \mathcal{H}_k$ such that:
\[
C_{X \mid A} \ell_a = \mu_{X \mid A=a}, \quad C_{X \mid A}^\top f = \mathbb{E}_X[f(X) \mid A = \cdot], \quad \forall f \in \mathcal{H}_k.
\]

We denote the sequentially collected online data as $\{(a_t, z_t)\}_{t=1}^T$, with
\[
\mathbf{a}_t := \begin{bmatrix} a_1 & \cdots & a_t \end{bmatrix}^\top, \quad \mathbf{z}_t := \begin{bmatrix} z_1 & \cdots & z_t \end{bmatrix}^\top,
\]
and the offline data $\{(x_j, a_j)\}_{j=1}^N$ as
\[
\mathbf{x} := \begin{bmatrix} x_1 & \cdots & x_N \end{bmatrix}^\top, \quad \mathbf{a} := \begin{bmatrix} a_1 & \cdots & a_N \end{bmatrix}^\top.
\]

Let $\mathbf{Q}_{\mathbf{aa}} := q(\mathbf{a}, \mathbf{a})$ denote the covariance of the conditional mean process $g(a)$. Then the joint distribution of $f(\mathbf{x})$ and $\mathbf{z}_t$ is:
\begin{align}
    \left[\begin{array}{c}
        f(\mathbf{x}) \\
        \mathbf{z}_t
    \end{array}\right]
    \mid \mathbf{a}, \mathbf{a}_t \sim \mathcal{N}\left(
    \left[\begin{array}{c}
        m(\mathbf{x}) \\
        \nu(\mathbf{a}_t)
    \end{array}\right],
    \left[\begin{array}{cc}
        \mathbf{K}_{\mathbf{xx}} & \boldsymbol{\Upsilon} \\
        \boldsymbol{\Upsilon}^\top & \mathbf{Q}_{\mathbf{a}_t \mathbf{a}_t} + \sigma^2 \mathbf{I}_t
    \end{array}\right]
    \right),
\end{align}
where the cross-covariance $\boldsymbol{\Upsilon} = \operatorname{Cov}(f(\mathbf{x}), \mathbf{z}_t) = \Phi_{\mathbf{x}}^\top C_{X \mid A} \Psi_{\mathbf{a}_t}$.
Given the offline data $[\mathbf{x}, \mathbf{a}]$, the posterior of the function $g$ evaluated at the test inputs $\mathbf{a}_t$ is given by:

\begin{align*}
    \nu_t(\mathbf{a}_t) =& \ell(\mathbf{a}_t, \mathbf{a}) \left( \mathbf{L}_{\mathbf{aa}} + N \lambda \mathbf{I}_N \right)^{-1} \Phi_{\mathbf{x}}^\top m, \\
    q_t(\mathbf{a}_t, \mathbf{a}_t) =& \ell(\mathbf{a}_t, \mathbf{a}) \left( \mathbf{L}_{\mathbf{aa}} + N \lambda \mathbf{I}_N \right)^{-1} \mathbf{K}_{\mathbf{xx}} \cdot \\
    & \left( \mathbf{L}_{\mathbf{aa}} + N \lambda \mathbf{I}_N \right)^{-1} \ell(\mathbf{a}, \mathbf{a}_t).
\end{align*}

Let $C_{AA} := \mathbb{E}[\ell_A \otimes \ell_A] \in \mathcal{H}_\ell \otimes \mathcal{H}_\ell$. Assume $\ell_a \in \operatorname{Range}(C_{AA})$ and $k_x \in \operatorname{Range}(C_{X \mid A} C_{AA} C_{X \mid A}^\top)$ for all $a \in \mathcal{A}$ and $x \in \mathcal{X}$. Then, conditioned on observations $\mathbf{z}_t$ with homoscedastic noise $\sigma^2$, the posterior distribution of $f$ is given by:
\begin{align*}
    m_t(x) &= m(x) + k(x, \mathbf{x}) \mathbf{W}_t \mathbf{M}^t_{\mathbf{a}_t \mathbf{a}_t} \left( \mathbf{z}_t - \nu(\mathbf{a}_t) \right), \\
    k_t(x, x') &= k(x, x') - k(x, \mathbf{x}) \mathbf{W}_t \mathbf{M}^t_{\mathbf{a}_t \mathbf{a}_t} \mathbf{W}_t^\top k(\mathbf{x}, x'),
\end{align*}
where
\begin{align*}
    \mathbf{M}^t_{\mathbf{a}_t \mathbf{a}_t} &= \left( q_t(\mathbf{a}_t \mathbf{a}_t) + \sigma^2 \mathbf{I}_t \right)^{-1}, \quad \mathbf{W}_t := (\mathbf{L}_{\mathbf{aa}} + N \lambda \mathbf{I})^{-1} \mathbf{L}_{\mathbf{a} \mathbf{a}_t}.
\end{align*}

\section{Alternative Approximation for\\ CMES Conditional Entropy}
\label{sec: CMES appendix}

In the main text, we approximate the second term in the CMES objective (Eq.~\ref{equ: DMES}) using the distribution \( p(z \mid z \leq f_*, \mathcal{D}^{t-1})\), which provides a computationally efficient surrogate to the intractable conditional \( p(z|f_*, a, \mathcal{D}^{t-1}) \). This approximation treats the thresholding directly in the observation space.
For brevity, we omit the conditioning on dataset $\mathcal{D}^{t-1}$ in the following expressions.

Here, we present an alternative, more principled formulation based on the distribution \( p(z \mid g(a) \leq f_*) \). This reflects the belief that the integrated signal \( g(a) \) serves as a surrogate for the unknown target function \( f(x) \), and conditioning on \( g(a) \leq f_* \) is more aligned with the structure of the underlying latent process. This alternative is not used in our experiments, but we include it here for completeness.

Due to the noise term, the density $p(z|g(a) \leq f_\ast)$ is not the truncated normal.
Using Bayes' theorem, we decompose this density as 
\begin{align*}
    p(z \mid g(a) \leq f_*) = \frac{p(g(a) \leq f_* \mid z) \cdot p(z \mid a)}{p(g(a) \leq f_*)}.
\end{align*}

The marginal terms are derived from the predictive distribution:
\begin{align*}
    p(z \mid a) &= \frac{\phi(\varphi_z(a))}{\sqrt{q_{t-1}(a,a) + \sigma^2}}, \\
    p(g(a) \leq f_*) &= \Phi(\gamma_{f_*}(a)),
\end{align*}
with $\gamma_{f_*}(a)$ defined in Eq. (\ref{equ: gamma}), and
\begin{align*}
    \varphi_z(a) := \frac{z - \nu_{t-1}(a)}{\sqrt{q_{t-1}(a,a) + \sigma^2}},
\end{align*}
where \( \phi \) and \( \Phi \) are the standard normal density and cumulative distribution functions.
To evaluate \( p(g(a) \leq f_* \mid z, a) \), we use the posterior distribution of \( g(a)| z, a \), which is Gaussian:
\[
g(a) \mid z, a \sim \mathcal{N}\left(u_{\text{noise}}(a), s^2_{\text{noise}}(a)\right),
\]
with
\begin{align*}
    u_{\text{noise}}(a) &= \frac{q_{t-1}(a,a)}{q_{t-1}(a,a) + \sigma^2} \left(z - \nu_{t-1}(a)\right) + \nu_{t-1}(a), \\
    s^2_{\text{noise}}(a) &= q_{t-1}(a,a) - \frac{q^2_{t-1}(a,a)}{q_{t-1}(a,a) + \sigma^2}.
\end{align*}
Thus,
\[
p(g(a) \leq f_* \mid z, a) = \Phi\left(\tilde{\gamma}_{f_*}(a)\right), \quad
\tilde{\gamma}_{f_*}(a) := \frac{f_* - u_{\text{noise}}(a)}{s_{\text{noise}}(a)}.
\]

The conditional entropy \( H(z \mid g(a) \leq f_*) \) can now be expressed as:
\begin{align*}
    H(z \mid g(a) \leq f_*) = - \int p(z \mid g(a) \leq f_*) \log p(z \mid g(a) \leq f_*) \, dz,
\end{align*}
which numerically evaluates to:
\begin{align*}
    - \int Z \Phi(\tilde{\gamma}_{f_*}(a)) \phi(\varphi_z(a)) \log \left[ Z \Phi(\tilde{\gamma}_{f_*}(a)) \phi(\varphi_z(a)) \right] dz,
\end{align*}
with normalization constant
\[
Z := \frac{1}{\sqrt{q_{t-1}(a,a) + \sigma^2} \cdot \Phi(\gamma^g_{f_*}(a))}.
\]

Accordingly, the CMES objective can be approximated by:
\begin{align*}
    \frac{1}{2} &\log\left[ 2\pi e \left(q_{t-1}(a,a) + \sigma^2\right) \right] ]]
    - \frac{1}{|\mathcal{F}_*|} \sum_{f_* \in \mathcal{F}_*} \sum_{z} \\
    & \Phi(\tilde{\gamma}_{f_*}(a)) \phi(\varphi_z(a)) 
    \log\left[ Z \Phi(\tilde{\gamma}_{f_*}(a)) \phi(\varphi_z(a)) \right],
\end{align*}
where the inner sum over \( z \) is computed numerically using quadrature or grid-based approximation.

\section{Experimental Settings}
\label{appendix: experimental settings}
We consider the negative Branin function and thus aim to maximise the function where the global maximum $f\left(x^*\right)= - 0.397887 \text {, at } x^*=(-\pi, 12.275),(\pi, 2.275) \text { and }(9.42478,2.475)$.
The function is evaluated on the square $x_1 \in [-5,10], x_2 \in [0,15]$.
We consider the query space $\mathcal{A} = [0,1]^2$. For linear transformation, we define $x_0 = 15 \times y_0 - 5 + \varepsilon; x_1 = 15 * y_1 + \varepsilon$; for non-transformation, we define $x_0 = 15 \times \cos(\pi y_0/2) - 5 + \varepsilon; x_1 = 15 \times \cos(\pi y_1/2) + \varepsilon$ (and truncated to be within the input space).

\section{Proof}
\label{sec: proof appendix}
We provide supplementary proofs for our theoretical results listed in Section \ref{sec: theory}.

\subsection{Proof: Regret Bound of CMES Policy}

\begin{proof}[Proof for Lemma \ref{lemma: equivalence among policies}]
The equivalent among MES, EST and GPUCB was proved in \cite{wang2017MES}. 
MES w.r.t. function $g$ is to select $a_{t} = {\arg\max}_{a \in \mathcal{A}} \mathbb{I}(z, g_\ast| a, \mathcal{D}^{t-1})$, where our policy can be obtained to replace $g_\ast$ to $f_\ast$. 
\end{proof}

\begin{remark}
    By replacing $g_\ast$ to $f_\ast$, CMES enables learning and optimising function values in mismatch spaces. 
    Note that the equivalence relies on our approximation in Eq. (\ref{equ: CMES approximation}) is a monotonic decreasing function w.r.t $\gamma_{f_\ast}(a)$. 
    Thus maximising Eq. (\ref{equ: CMES approximation}) is equivalent to minimising $\gamma_{f_\ast}(a)$.
    And EST with $m = f_\ast$ is to select data point that minimise $\gamma_{f_\ast}(a)$.
    However, it is unclear whether this equivalence still holds if we change the approximation method, for example in Appendix \ref{sec: CMES appendix}.  
\end{remark}

\begin{proof}[Proof for Lemma \ref{lemma: EST regret bound}]
From the definition of $\kappa_t$, we have
\begin{align*}
    \nu_{t-1}(a_t) + \kappa_t \sqrt{q_{t-1}(a_t, a_t)} \geq \hat{m}_t \geq f_\ast 
\end{align*}
Then with Lemma \ref{lemma: concentration g EST}, we have 
    \begin{align*}
        \tilde{r}^g_t &= \max_{x \in \mathcal{X}} f(x) - g(a_t)\\
        & \leq \nu_{t-1}(a_t) + \kappa_t \sqrt{q_{t-1}(a_t, a_t)} - g(a_t)\\
        & \leq \left(\kappa_t+\zeta_t\right) \sqrt{q_{t-1}(a_t, a_t)}
    \end{align*}
\end{proof}

\begin{proof}[Proof of Lemma \ref{lemma: information gain - var}]
    Note $\mathbb{I}(\bm{z}_T; f) = H(\bm{z}_T) - H(\bm{z}_T|f)$.
Since $z(a)=\int f(x)p(x|a)dx+\epsilon$, with $\epsilon\sim \mathcal{N}\left(0,\sigma^{2}\right)$, we have $z(a)|f\sim \mathcal{N}\left(\int f(x)p(x|a)dx,\sigma^{2}\right)$. Hence  $H(\bm{z}_T|f) = \frac{1}{2}\log |2\pi e\sigma^{2}\bm{I}|$, which does not depend on $f$.  

On the other hand, $H(z_T|\mathbf{z}_{T-1}) = \frac{1}{2}\log (2\pi e (\sigma^2  + q_{T-1}(a_T, a_T))$.
Following Eq. (\ref{equ: DMES}), $a_1, \dots, a_T$ are deterministic conditioned on $\bm{z}_{T-1}$, and since the conditional variance $q_{T-1}(a_T, a_T)$ does not depend on $\bm{z}_{T-1}$, we have
$H(\bm{z}_T) = H(\bm{z}_{T-1}) + H(z_T| \bm{z}_{T-1}) = H(\bm{z}_{T-1}) + \frac{1}{2} \log (2 \pi e (\sigma^2 + q_{T-1}(a_T, a_T)))$.
By induction, the proof concludes. 
\end{proof}

\begin{proof}[Proof for Theorem \ref{theo: Instant Regret Bound CMES}]
Following \cite{wang2017MES}, we know there exists at least one $f_*^t$ that satisfies $f_*^t>f_*$ with probability at least $1-w^{k_i}$ in $k_i$ iterations.
Then with $T^{\prime}=\sum_{i=1}^T k_i$ as the total number of iterations, we split these iterations into $T$ parts where each part has $k_i$ iterations, $i \in [1, T]$. With probability at least $1-\sum_{i=1}^T w^{k_i}$, we have at least one iteration $t_i$ which samples $f_*^{t_i}$ satisfying $f_*^{t_i}>f_*, \forall i=1, \cdots, T$ by union bound.

Let $\sum_{i=1}^T w^{k_i}=\frac{\delta}{2}$, we can set $k_i=\log _w \frac{\delta}{2 \pi_i}$ for any $\sum_{i=1}^T\left(\pi_i\right)^{-1}=1$. Following \cite{srinivas_gpucb:information-theoretic_2012}, we choose $\pi_i$ is $\pi_i=\frac{\pi^2 i^2}{6}$. 
Hence with probability at least $1-\frac{\delta}{2}$, there exist a sampled $f_*^{t_i}$ satisfying $f_*^{t_i}>f_*, \forall i=1, \cdots, T$

From Lemma \ref{lemma: EST regret bound} and \ref{lemma: information gain - var}, we have 
\begin{align*}
    \sum_{t_i=1}^T{(\tilde{r}^g_{t_i})}^2 &\leq \sum_{t_i=1}^T \left(\kappa_{t_i}+\zeta_{t_i}\right)^2 q_{t_i-1}\left(a_{t_i}\right)\\
    & \leq  \sum_{t_i=1}^T \left(\kappa_{t_i}+\zeta_{t_i}\right)^2 \sigma^2 (\sigma^{-2} q_{t_i-1}(a_{t_i}, a_{t_i}))\\
    & \leq 2 \left(\kappa_{t^*}+\zeta_{T}\right)^2 \sigma^2 C_1 \rho_T \\
    & = C \left(\kappa_{t^*}+\zeta_{T}\right)^2 \rho_T
\end{align*}
   where $C = 2 \sigma^2 C_1 = 2/\log(1 + \sigma^{-2}) \geq 2 \sigma^2$, $C_1 = \sigma^{-2} / \log \left(1+\sigma^{-2}\right) \geq 1$, since $s^2 \leq C_1 \log(1 + s^2)$ for $s^2 \in [0, \sigma^{-2}]$. Then by the Cauchy-Schwarz inequality. 
\begin{align*}
    \sum_{i=1}^T \tilde{r}^g_{t_i} &\leq \left(\kappa_{t^*}+\zeta_{T}\right) \sqrt{T \sum_{t=1}^T {(\tilde{r}^g_t)}^2} \\
    & \leq \left(\kappa_{t^*}+\zeta_{T}\right)  \sqrt{T C \rho_T}
\end{align*}
Then the instant regret (Definition \ref{defi: instant regret}),
\begin{align*}
    r^g_{T^\prime} &= f(x_{*})- \max_{t \in [1,T]} g(a_{t})\\
    & \leq \frac{1}{T}\sum_{i=1}^T \tilde{r}^g_{t_i} = \left(\kappa_{t^*}+\zeta_{T}\right)\sqrt{C \rho_T/T} 
\end{align*}
where $T^{\prime}=\sum_{i=1}^T k_i=\sum_{i=1}^T \log _w \frac{\delta}{2 \pi_i}$ is the total number of iterations.
\end{proof}

\subsection{Proof: Regret of Multi-resolution Queries}

Assuming the conditional distribution $p(x|a)$ is known, we show the concentration inequality for the predictor CMP shown in Appendix \ref{appendix: CMP}, by following the concentration inequality of Gaussian distributed variables \citep{srinivas_gpucb:information-theoretic_2012}.
This is because  the posterior of $f(x)|\mathbf{z}_t \sim \mathcal{N}(m_t(x), \sigma^2_t(x))$.
We can further extend the concentration results to asymptotic bounds for empirical estimates when we need to estimate $p(x|a)$ from data, following the convergence result in \citep{chau_deconditional_2021}.

\begin{lemma}[Concentration inequality of $f$]
\label{theo: concentration}
Pick $\delta \in(0,1)$ and set $\beta_t=2 \log \left(|\mathcal{X}| \pi_t / \delta\right)$, where $\sum_{t \geq 1} \pi_t^{-1}=1, \pi_t>0$. Then, given $ \mathcal{D}^{t-1}$,
$$
\left|f(x)-m_{t-1}(x)\right| \leq \sqrt{\beta_t k_{t-1}(x,x)} \quad \forall x \in \mathcal{X}, \forall t \geq 1
$$
holds with probability $\geq 1-\delta$.
\end{lemma}

\begin{proof}[Proof for Lemma \ref{theo: concentration}]
From Appendix \ref{appendix: CMP},
we know 
$f(x)|\mathbf{z}_t \sim \mathcal{N}(m_t(x), \sigma^2_t(x))$, where $\sigma_t(x) = k_t(x,x)$. If $r \sim \mathcal{N}(0,1)$, then
\begin{align*}
    \begin{aligned}
\operatorname{Pr}\{r>c\} & =e^{-c^2 / 2}(2 \pi)^{-1 / 2} \int_c^{\infty} e^{-(r-c)^2 / 2-c(r-c)} d r \\
& \leq e^{-c^2 / 2} \operatorname{Pr}\{r>0\}=(1 / 2) e^{-c^2 / 2}
\end{aligned}
\end{align*}
Let $r = (f(x) - m_t(x))/\sigma_t(x)$ and $c = \sqrt{\beta_t}$.
Applying the union bound, we have 
\begin{align}
    \left|f(x)-m_{t}(x)\right| \leq \beta_t^{1 / 2} \sigma_{t}(x) \quad \forall x \in \mathcal{X}
\end{align}
holds with probability $\geq 1-|\mathcal{X}| e^{-\beta_t / 2}$.
\end{proof}

The proof for Lemma \ref{lemma: concentration g EST} follows the same logic, while we apply a union bound over the evaluated points $\{a_t\}_{t=1}^T$ following \cite{pmlr-v51-wang16f}.

\begin{proof}[Proof for Theorem \ref{theo: instant regret bound gpucb}]
Define 
$r^\prime_{t}=f(x^{*})- g(\xi_{t},\delta_t).
$
We start by the triangle inequality
\[
r^\prime_{t}\leq |f(x^{*})-f(\xi_{t})|+|f(\xi_{t})-g(\xi_{t},\delta_t)|.
\]
The first term is the usual regret. Assume we have at iteration $t-1$ the posterior of $f$ with mean $\mu_{t-1}$ and covariance $k_{t-1}$. 
Now, we can bound $|f(x^{*})-f(\xi_{t})|\leq 2\sqrt{ \beta_{t} k_{t-1}(\xi_{t},\xi_t)}$ in high probability.

The second term is $O(\delta_{t}^2)$ by Taylor expansion. We see this from
\begin{align*}
f(x)=f(\xi_t + \eta_t) &= f(\xi_t) + \nabla f(\xi_t)^\top \eta_t \\&+ \frac{1}{2} \eta_t^\top \nabla^2 f(\xi_t) \eta_t + O(\|\eta_t\|^3).
\end{align*}
Taking expectations over $\eta_t\sim p_t$ and using the third moment assumption on $\eta_t$ gives $\mathbb E\|\eta_t\|^3\leq C\delta_t^3$ so
\[
g(a_t) = f(\xi_t) + \frac{1}{2} \delta_{t}^2 \mathrm{Tr}(\nabla^2 f(\xi_t)) + o(\delta_{t}^2).
\]
Hence,  since $f$ has a bounded second derivative, i.e. $\mathrm{Tr}(\nabla^2 f(x))\leq 2M$, for small enough $\delta_{t}$, $|f(\xi_{t})-g(a_{t})|\leq M\delta_{t}^{2}$.

Now using $(a+b)^{2}\leq2a^2+2b^{2}$, we get
\begin{align*}
\sum_{t=1}^{T} {r_{t}^\prime}^{2}&\leq {2}\sum (|f(x^{*})-f(\xi_{t})|^{2}+|f(\xi_{t})-g(a_{t})|^{2})\\
&\leq 8\sum\beta_{t}k_{t-1}(\xi_{t}, \xi_{t})+2M\sum\delta_{t}^{4}.
\end{align*}

Considering the first term, we get
\begin{align*}
\sum_{t=1}^T\beta_t k_{t-1}(\xi_{t}, \xi_{t})
 & \leq \beta_T \sigma^2\sum_{t=1}^T\sigma^{-2} k_{t-1}(\xi_{t}, \xi_{t}) \\
& \leq \beta_T \sigma^2 C \sum_{t=1}^T\log \left(1+\sigma^{-2} k_{t-1}(\xi_{t}, \xi_{t})\right)\\
& =  2\beta_T \sigma^2 C \mathbb{I}(\mathbf{f}_\Xi+\boldsymbol{\epsilon},\mathbf{f}_\Xi),
\end{align*}
by Lemma 5.3 in \citet{srinivas_gpucb:information-theoretic_2012}, and we chose $C=\sigma^{-2} / \log \left(1+\sigma^{-2}\right) \geq 1$ and denoted $\mathbf{f}_\Xi=[f(\xi_t)]_{t=1}^T$ and $\boldsymbol{\epsilon}\sim \mathcal{N}(0,\sigma^2 I_T)$. 

Now we can bound $\mathbb{I}(\mathbf{f}_\Xi+\boldsymbol{\epsilon},\mathbf{f}_\Xi)$ by $\gamma_T$, the maximal information gain for \emph{direct queries} (i.e. $\delta_t=0$) given by
$$
    \gamma_T := \sup_{S \subset \mathcal{X}: |S| = T} \mathbb{I}(\mathbf{f}_S+\boldsymbol{\epsilon}, \mathbf{f}_S)= \sup_{S \subset \mathcal{X}: |S| = T}  \frac{1}{2} \log|I+\sigma^{-2}K_S|
$$
where $K_S$ is the covariance matrix of $\mathbf{f}_S=[f(x)]_{x\in S}$ and $\boldsymbol{\epsilon}\sim \mathcal{N}(0,\sigma^2 I_T)$. Quantity $\gamma_T$ was thoroughly analysed in \citet{srinivas_gpucb:information-theoretic_2012} for various choices of kernels $k$. Putting everything together gives

$$\sum_{t=1}^{T} {r_{t}^\prime}^{2}\leq O(\beta_{T}\gamma_{T})+O\left(\sum_{t=1}^T\delta_{t}^{4}\right).$$

Now, from Cauchy-Schwarz,
\[
\sum_{t=1}^{T} r^\prime_{t}\cdot 1\leq \sqrt{ T\sum_{t=1}^{T} {r_{t}^\prime}^{2} }=\mathcal{O}\left(\sqrt{ T\beta_{T}\gamma_{T}+T\sum_{t=1}^{T}\delta_{t}^{4} }\right)
\]

And finally the instant regret can be bounded as
\[
r^g_T \leq \frac{1}{T} \sum_{t=1}^{T} r^\prime_{t} = \mathcal{O} \left(\sqrt{\frac{\beta_T \gamma_T + \sum_{t=1}^{T}\delta_{t}^{4}}{T}}\right).
\]

From \citet{srinivas_gpucb:information-theoretic_2012}, the case of Gaussian kernel gives $\gamma_T=\mathcal{O}((\log T)^{d+1})$ and directly setting schedule for $\delta_t$ to match the rate of the first term gives the result.

\end{proof}

\begin{figure}[t!]
    \centering
    \includegraphics[width=0.8\linewidth]{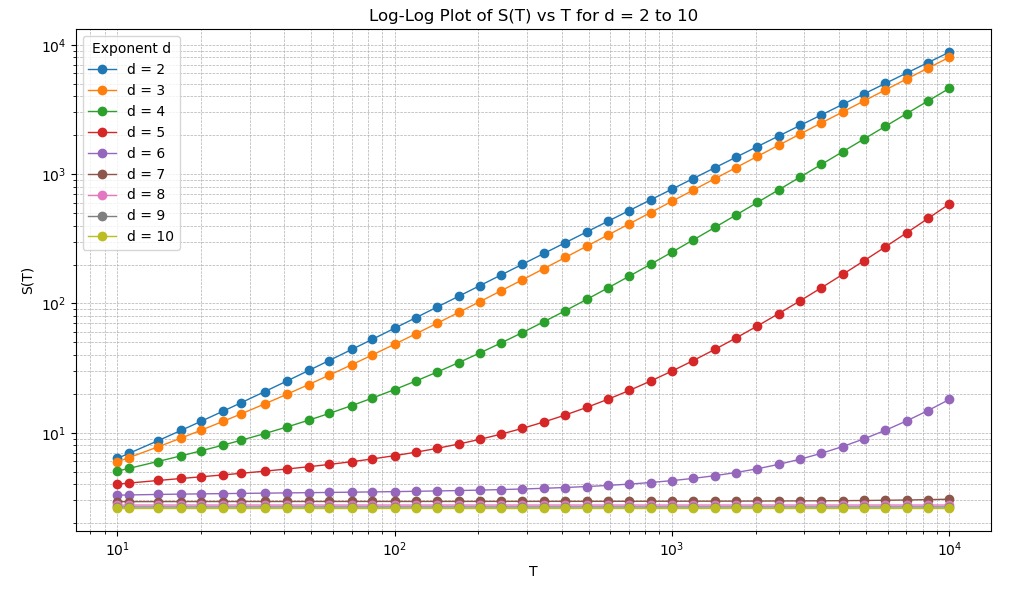}
    \caption{Savings by allowing multiresolution queries for varying dimensions.}
    \label{fig:savings_multiresolution}
\end{figure}

We further illustrate of cost savings in multiresolution queries in Figure \ref{fig:savings_multiresolution}.
Consider the case of a Gaussian kernel prior on $f$ and the setting where $g(a_t) = \int f(\xi_t + \eta_t) p_t(\eta_t) d\eta_t$ and $\eta_t$ is uniform on $[-\tau_{t},\tau_{t}]^{d}$, where $\tau_{t}=t^{-1/4}(\log t)^{d/4}$ follows a deterministic schedule. By Theorem \ref{theo: instant regret bound gpucb}, this schedule matches the regret upper bounds of the direct queries.

Assume that the cost associated to the "resolution" $\tau$ is given by

$$ \lambda(\tau) = \frac{1}{1+\tau^d}. $$

Note that this means that a direct query which has $\tau=0$ corresponds to the cost $1$ and $\tau>0$ implies a cheaper query.
The total cost of direct queries is thus $T$ and of our particular schedule $S(T)=\sum_{t=1}^T \frac{1}{1+(t^{-1/4}(\log t)^{d/4})^d}$. As the dimension $d$ increases, we obtain higher savings due to allowing direct queries.



%% file: main.bbl
\begin{thebibliography}{27}
\providecommand{\natexlab}[1]{#1}

\bibitem[{Audibert and Bubeck(2010)}]{audibert_best_2010}
Audibert, J.-Y.; and Bubeck, S. 2010.
\newblock Best {Arm} {Identification} in {Multi}-{Armed} {Bandits}.
\newblock 14.

\bibitem[{Avadhanula et~al.(2021)Avadhanula, Colini~Baldeschi, Leonardi,
  Sankararaman, and Schrijvers}]{avadhanulaStochasticBanditsMultiplatform2021}
Avadhanula, V.; Colini~Baldeschi, R.; Leonardi, S.; Sankararaman, K.~A.; and
  Schrijvers, O. 2021.
\newblock Stochastic Bandits for Multi-Platform Budget Optimization in Online
  Advertising.
\newblock In \emph{Proceedings of the {{Web Conference}} 2021}, {{WWW}} '21,
  2805--2817.

\bibitem[{Aziz, Kaufmann, and Riviere(2021)}]{azizMultiArmedBanditDesigns2021}
Aziz, M.; Kaufmann, E.; and Riviere, M.-K. 2021.
\newblock On {{Multi-Armed Bandit Designs}} for {{Dose-Finding Trials}}.
\newblock \emph{Journal of Machine Learning Research}, 22(14): 1--38.

\bibitem[{Chau, Bouabid, and Sejdinovic(2021)}]{chau_deconditional_2021}
Chau, S.~L.; Bouabid, S.; and Sejdinovic, D. 2021.
\newblock Deconditional {Downscaling} with {Gaussian} {Processes}.
\newblock In \emph{Advances in {Neural} {Information} {Processing} {Systems}},
  volume~34, 17813--17825.

\bibitem[{Chowdhury, Oliveira, and Ramos(2020)}]{chowdhury2020active}
Chowdhury, S.~R.; Oliveira, R.; and Ramos, F. 2020.
\newblock Active learning of conditional mean embeddings via bayesian
  optimisation.
\newblock In \emph{Conference on Uncertainty in Artificial Intelligence},
  1119--1128. PMLR.

\bibitem[{Fiez et~al.(2019)Fiez, Jain, Jamieson, and
  Ratliff}]{fiez2019sequentialexperimentaldesigntransductive}
Fiez, T.; Jain, L.; Jamieson, K.; and Ratliff, L. 2019.
\newblock Sequential Experimental Design for Transductive Linear Bandits.
\newblock arXiv:1906.08399.

\bibitem[{Garcia-Barcos and
  Martinez-Cantin(2020)}]{garciabarcos2020robustpolicysearchrobot}
Garcia-Barcos, J.; and Martinez-Cantin, R. 2020.
\newblock Robust Policy Search for Robot Navigation with Stochastic
  Meta-Policies.
\newblock arXiv:2003.01000.

\bibitem[{Garnett(2023)}]{garnett_bayesopIQBOok_2023}
Garnett, R. 2023.
\newblock \emph{{Bayesian Optimization}}.
\newblock Cambridge University Press.

\bibitem[{Hartford et~al.(2017)Hartford, Lewis, Leyton-Brown, and
  Taddy}]{pmlr-v70-hartford17a}
Hartford, J.; Lewis, G.; Leyton-Brown, K.; and Taddy, M. 2017.
\newblock Deep {IV}: A Flexible Approach for Counterfactual Prediction.
\newblock In \emph{Proceedings of the 34th International Conference on Machine
  Learning}, 1414--1423.

\bibitem[{Hennig and Schuler(2012)}]{hennig2012entropy_search}
Hennig, P.; and Schuler, C.~J. 2012.
\newblock Entropy search for information-efficient global optimization.
\newblock \emph{Journal of Machine Learning Research}, 13(6).

\bibitem[{Kandasamy et~al.(2016)Kandasamy, Dasarathy, Oliva, Schneider, and
  P{\'o}czos}]{kandasamy2016gaussian}
Kandasamy, K.; Dasarathy, G.; Oliva, J.~B.; Schneider, J.; and P{\'o}czos, B.
  2016.
\newblock Gaussian process bandit optimisation with multi-fidelity evaluations.
\newblock \emph{Advances in neural information processing systems}, 29.

\bibitem[{Kandasamy et~al.(2017)Kandasamy, Dasarathy, Schneider, and
  Poczos}]{kandasamy2017multifidelitybayesianoptimisationcontinuous}
Kandasamy, K.; Dasarathy, G.; Schneider, J.; and Poczos, B. 2017.
\newblock Multi-fidelity Bayesian Optimisation with Continuous Approximations.
\newblock arXiv:1703.06240.

\bibitem[{Kirschner, Lattimore, and
  Krause(2023)}]{kirschner2023linearpartialmonitoringsequential}
Kirschner, J.; Lattimore, T.; and Krause, A. 2023.
\newblock Linear Partial Monitoring for Sequential Decision-Making: Algorithms,
  Regret Bounds and Applications.
\newblock arXiv:2302.03683.

\bibitem[{Law et~al.(2018)Law, Sejdinovic, Cameron, Lucas, Flaxman, Battle, and
  Fukumizu}]{law_variational_2018}
Law, H.~C.; Sejdinovic, D.; Cameron, E.; Lucas, T.; Flaxman, S.; Battle, K.;
  and Fukumizu, K. 2018.
\newblock Variational {Learning} on {Aggregate} {Outputs} with {Gaussian}
  {Processes}.
\newblock In \emph{Advances in {Neural} {Information} {Processing} {Systems}},
  volume~31.

\bibitem[{Muandet et~al.(2017)Muandet, Fukumizu, Sriperumbudur, and
  Schölkopf}]{muandet_kernel_2017}
Muandet, K.; Fukumizu, K.; Sriperumbudur, B.; and Schölkopf, B. 2017.
\newblock Kernel {Mean} {Embedding} of {Distributions}: {A} {Review} and
  {Beyond}.
\newblock \emph{Foundations and Trends® in Machine Learning}, 10(1-2): 1--141.
\newblock ArXiv:1605.09522 [cs, stat].

\bibitem[{Mutný and
  Krause(2023)}]{mutný2023experimentaldesignlinearfunctionals}
Mutný, M.; and Krause, A. 2023.
\newblock Experimental Design for Linear Functionals in Reproducing Kernel
  Hilbert Spaces.
\newblock arXiv:2205.13627.

\bibitem[{Oliveira, Ott, and
  Ramos(2019)}]{oliveira2019bayesianoptimisationuncertaininputs}
Oliveira, R.; Ott, L.; and Ramos, F. 2019.
\newblock Bayesian optimisation under uncertain inputs.
\newblock arXiv:1902.07908.

\bibitem[{Rasmussen and Williams(2006)}]{rasmussen_gaussian_2006}
Rasmussen, C.~E.; and Williams, C. K.~I. 2006.
\newblock \emph{Gaussian processes for machine learning}.
\newblock Adaptive computation and machine learning. Cambridge, Mass: MIT
  Press.
\newblock ISBN 978-0-262-18253-9.
\newblock OCLC: ocm61285753.

\bibitem[{Song et~al.(2009)Song, Huang, Smola, and
  Fukumizu}]{song_hilbert_2009}
Song, L.; Huang, J.; Smola, A.; and Fukumizu, K. 2009.
\newblock Hilbert space embeddings of conditional distributions with
  applications to dynamical systems.
\newblock In \emph{Proceedings of the 26th {Annual} {International}
  {Conference} on {Machine} {Learning}}, 961--968. Montreal Quebec Canada: ACM.
\newblock ISBN 978-1-60558-516-1.

\bibitem[{Srinivas et~al.(2012)Srinivas, Krause, Kakade, and
  Seeger}]{srinivas_gpucb:information-theoretic_2012}
Srinivas, N.; Krause, A.; Kakade, S.~M.; and Seeger, M.~W. 2012.
\newblock {GPUCB}: {Information}-{Theoretic} {Regret} {Bounds} for {Gaussian}
  {Process} {Optimization} in the {Bandit} {Setting}.
\newblock \emph{IEEE Transactions on Information Theory}, 58(5): 3250--3265.

\bibitem[{Stanton et~al.(2022)Stanton, Maddox, Gruver, Maffettone, Delaney,
  Greenside, and Wilson}]{pmlr-v162-stanton22a}
Stanton, S.; Maddox, W.; Gruver, N.; Maffettone, P.; Delaney, E.; Greenside,
  P.; and Wilson, A.~G. 2022.
\newblock Accelerating {B}ayesian Optimization for Biological Sequence Design
  with Denoising Autoencoders.
\newblock In \emph{International Conference on Machine Learning}, 20459--20478.

\bibitem[{Takeno et~al.(2020)Takeno, Fukuoka, Tsukada, Koyama, Shiga, Takeuchi,
  and Karasuyama}]{takeno2020MFMES}
Takeno, S.; Fukuoka, H.; Tsukada, Y.; Koyama, T.; Shiga, M.; Takeuchi, I.; and
  Karasuyama, M. 2020.
\newblock Multi-fidelity Bayesian optimization with max-value entropy search
  and its parallelization.
\newblock In \emph{International Conference on Machine Learning}, 9334--9345.

\bibitem[{Wang and Jegelka(2017)}]{wang2017MES}
Wang, Z.; and Jegelka, S. 2017.
\newblock Max-value entropy search for efficient Bayesian optimization.
\newblock In \emph{International Conference on Machine Learning}, 3627--3635.
  PMLR.

\bibitem[{Wang, Zhou, and Jegelka(2016)}]{pmlr-v51-wang16f}
Wang, Z.; Zhou, B.; and Jegelka, S. 2016.
\newblock Optimization as Estimation with Gaussian Processes in Bandit
  Settings.
\newblock In \emph{Proceedings of the 19th International Conference on
  Artificial Intelligence and Statistics}, 1022--1031.

\bibitem[{Wu et~al.(2019{\natexlab{a}})Wu, Chen, Zhang, Xiong, Lei, and
  Deng}]{wu2019hyperparameter}
Wu, J.; Chen, X.-Y.; Zhang, H.; Xiong, L.-D.; Lei, H.; and Deng, S.-H.
  2019{\natexlab{a}}.
\newblock Hyperparameter optimization for machine learning models based on
  Bayesian optimization.
\newblock \emph{Journal of Electronic Science and Technology}, 17(1): 26--40.

\bibitem[{Wu et~al.(2019{\natexlab{b}})Wu, Toscano-Palmerin, Frazier, and
  Wilson}]{wu2019practicalmultifidelitybayesianoptimization}
Wu, J.; Toscano-Palmerin, S.; Frazier, P.~I.; and Wilson, A.~G.
  2019{\natexlab{b}}.
\newblock Practical Multi-fidelity Bayesian Optimization for Hyperparameter
  Tuning.
\newblock arXiv:1903.04703.

\bibitem[{Zhang, Tsuchida, and Ong(2022)}]{zhang2022gaussian}
Zhang, M.; Tsuchida, R.; and Ong, C.~S. 2022.
\newblock Gaussian process bandits with aggregated feedback.
\newblock In \emph{Proceedings of the AAAI Conference on Artificial
  Intelligence}, volume~36, 9074--9081.

\end{thebibliography}
